\documentclass[12pt]{article}
\usepackage{pdfpages,amsmath,amsthm,amssymb,amsopn,amsfonts}
\usepackage{graphics}
\usepackage{graphicx,psfrag,epsf,subfigure}
\usepackage{enumerate}
\usepackage{natbib}
\usepackage{relsize}
\usepackage{url} 
\usepackage{bbm}
\usepackage{caption}
\usepackage{algorithmic}
\usepackage{algorithm}

\newcommand{\blind}{0}

\newcommand{\gt}{\tau}

\newcommand{\D}{\Delta}
\newcommand{\tr}{\textup{trace}}

\newcommand{\bl}{{\bf L}}
\newcommand{\bX}{{\bf X}}
\newcommand{\bG}{{\bf G}}
\newcommand{\bW}{{\bf W}}
\newcommand{\by}{{\bf y}}
\newcommand{\bZ}{{\bf Z}}

\newtheorem{theorem}{Theorem}
\newtheorem{proposition}[theorem]{Proposition}

\newtheorem{corollary}[theorem]{Corollary}
\newtheorem{remark}[theorem]{Remark}

\begin{document}

\def\spacingset#1{\renewcommand{\baselinestretch}%
{#1}\small\normalsize} \spacingset{1}


\if0\blind
{
  \title{\bf Fast Embedding for JOFC Using the Raw Stress Criterion}
  \author{Vince Lyzinski\thanks{
    The authors gratefully acknowledge support from 
the XDATA program of the Defense Advanced
Research Projects Agency (DARPA) administered
through Air Force Research Laboratory contract FA8750-12-2-0303,
and the NSF BRAIN Early Concept Grants for Exploratory Research (EAGER) award DBI-1451081.}\hspace{.2cm}\\
    Johns Hopkins University Human Language Technology Center of Excellence,\\
    Youngser Park\\
    Center for Imaging Sciences, Johns Hopkins University,\\
    Carey E. Priebe\\
    Department of Applied Mathematics and Statistics, Johns Hopkins University,\\
    and\\
    Michael Trosset\\
    Department of Statistics, Indiana University}
  \maketitle
} \fi

\if1\blind
{
  \bigskip
  \bigskip
  \bigskip
  \begin{center}
    {\LARGE\bf Fast Embedding for JOFC Using the Raw Stress Criterion}
\end{center}
  \medskip
} \fi

\bigskip
\begin{abstract}
The Joint Optimization of Fidelity and Commensurability (JOFC)
manifold matching methodology
embeds an omnibus dissimilarity matrix consisting of
multiple dissimilarities on the same set of objects.
One approach to this embedding
optimizes the preservation of fidelity to each individual dissimilarity matrix
together with commensurability of each given observation across modalities
via iterative majorization of a raw stress error criterion by successive Guttman transforms. 
In this paper, we exploit the special structure inherent to JOFC 
to exactly and efficiently compute the successive Guttman transforms, 
and as a result we are able to greatly speed up the JOFC procedure for both in-sample and out-of-sample embedding.  
We demonstrate the scalability of our implementation on both real and simulated data examples.
\end{abstract}

\noindent%
{\it Keywords:}  Multidimensional Scaling, manifold matching, distance geometry, distance matrices
\vfill

\newpage
\spacingset{1.45} 
\section{Introduction and Background}
\label{S:intro}

Manifold matching---embedding multiple modality data sets into a common low-dimensional space wherein joint inference can be investigated---is an important inference task in statistical pattern recognition, with applications in computer vision \citep[see, for example,][]{MMCV1,MMCV2,MMCV3,MMCV4,ham2003learning}, text and language processing \citep[see, for example,][]{MMHLT1,MMHLT2,MMHLT3}, and machine learning \citep[see, for example,][]{MMML1,MMML2,MMML3,ham2005semisupervised}, to name a few; for a survey of the literature on manifold matching and the broader problem of transfer learning, see \cite{pan2010survey}.

In the present manifold matching framework, we consider $n$ objects, each measured under $m$ disparate modalities or conditions, each modality yielding an object-wise dissimilarity matrix $\{\D_i\}_{i=1}^m$; thus $\D_1,\D_2,\ldots,\D_m\in\mathbb{R}_+^{n\times n}$.  
The Joint Optimization of Fidelity and Commensurability (JOFC) algorithm of \cite{JOFC} is a manifold matching procedure that simultaneously embeds these $mn$ data points ($n$ objects in $m$ modalities) into a common Euclidean space by embedding an omnibus dissimilarity matrix ${\bf \Delta}$ which encapsulates the information contained in the dissimilarities $\{\D_i\}_{i=1}^m$.
The JOFC algorithm has proven to be a flexible and effective manifold matching algorithm, with numerous applications and extensions in the literature; see \cite{ma2012fusion,sunpriebe2013,sgmjofc,adali2013fidelity,shen2014manifold}.
One approach to this embedding
optimizes the preservation of fidelity to each individual dissimilarity matrix (i.e., preserving the within modality dissimilarities)
together with the commensurability of the observations across modalities (i.e., preserving the cross-modality matchedness of the data).  
This approach embeds ${\bf \D}$ by minimizing Kruskal's raw stress criterion for metric multidimensional scaling (MDS) via successive Guttman transforms \citep{borg2005modern};
see Algorithm \ref{alg:1}.  

In this paper, we exploit the special structure of the JOFC weight matrix to exactly and efficiently compute these successive Guttman transforms.  
Employing this and further computational simplifications, we are able to dramatically speed the JOFC procedure (see Algorithm \ref{alg:2}) and extend this speedup to out-of-sample embedding for JOFC.  
In addition, parallelizing the resulting algorithm---see Remark \ref{rem:parfjofc}---is immediate.
We demonstrate these speedups and the utility of the JOFC framework in real and synthetic data examples.

\noindent{\bf Notation:} To aid the reader, we have collected the frequently used notation introduced in this manuscript into a table for ease of reference; see Table \ref{table:not}.
\begin{table}[t!]
\hspace{-10mm}
\begin{tabular}{| c|c |c |}
\hline
   Notation & Description &  Reference\\
  \hline\hline
  $J_n$& The $n\times n$ matrix of all $1$'s & Used throughout\\
  \hline
  $I_n$& The $n\times n$ identity matrix & Used throughout\\
  \hline
  \bf X & Final configuration obtained via 3-RSMDS JOFC and fJOFC & Sec. \ref{S:3SMDS}, \ref{S:JOFC} and \ref{sec:fastJOFC} \\
  \hline
   $\tilde\sigma(\bf X)$ & The raw stress objective fcn. of 3-RSMDS & Eq. (\ref{eq:3waystress}) \\
   \hline
   $\widetilde{\bf \Delta}$ & The omnibus dissimilarity embedded by 3-RSMDS & Eq. (\ref{eq:rsdelta})\\
   \hline
  $\sigma(\bf X)$ & The raw stress objective fcn. of JOFC and fJOFC & Eq. (\ref{stress}) \\
  \hline
   ${\bf \Delta}$ & The omnibus dissimilarity embedded by JOFC and fJOFC & Eq. (\ref{eq:jofcdelta})\\
  \hline
  $\bf W$ & The weight matrix used in the JOFC and fJOFC embeddings  & Eq. (\ref{eq:W})\\
  \hline
    $\bf L$ & The combinatorial Laplacian of $\bf W$  & Sec. \ref{S:JOFC} and \ref{sec:fastJOFC}\\
    \hline
    $B({\bf X})$& The $B$-matrix used in the JOFC Guttman transform updates & Eq. (\ref{eq:B})\\
  \hline
  $\bf L^\dagger$& The Moore-Penrose pseudoinverse of $\bf L$& Sec. \ref{S:JOFC} and \ref{sec:fastJOFC}\\
   \hline
   $\mathcal W$ & A modified weight matrix used in the computation of $\bf L^\dagger$  & Eq. (\ref{eq:modW})\\
    \hline
   ${\bf \Delta}^{(o)}$ & The out-of-sample omnibus dissimilarity embedded by fJOFC & Sec. \ref{S:foos}\\
   \hline
      $\sigma_{\bf X}({\bf y})$ & The out-of-sample raw stress criterion & Sec. \ref{S:foos}\\
      \hline
         ${\bf W}^{(o)}$ & The out-of-sample weight matrix used in the fJOFC embedding & Sec. \ref{S:foos}\\
   \hline
         ${\bf L}^{(o)}$ & The combinatorial Laplacian of ${\bf W}^{(o)}$ & Sec. \ref{S:foos}\\
         \hline
\end{tabular}
\caption{Table of relevant notation.}
\label{table:not}
\end{table}

\subsection{JOFC and Three-Way Raw Stress MDS}
\label{S:background}
In the JOFC framework, we use Raw Stress MDS to simultaneously embed the $m$ object-wise dissimilarity matrices $\D_1,\D_2,\ldots,\D_m\in\mathbb{R}^{n\times n}_+$ while preserving both the matchedness of the objects across modality and the within modality dissimilarities.
In this way, JOFC is closely related to Three-Way Raw Stress MDS (3-RSMDS).
The key difference is that the cross modality matchedness of the objects in 3-RSMDS is enforced via a constraint on the feasible region,
while in JOFC the matchedness is enforced by adding a suitable term into the raw stress criterion. 
In that light, 
JOFC can be viewed as a softly constrained version of 3-RSMDS.
We highlight the commonalities and differences between the two approaches below, and in Section \ref{S:3J} empirically compare their respective performances in an illustrative simulation.  
For further discussion of the connection between JOFC and Three-Way Nonmetric MDS in the context of hypothesis testing, see \cite{castle2012quasi}, Chapter 8.

\begin{remark}
\emph{While the JOFC algorithm is closely related to 3-RSMDS, it bears mentioning the relationship of the algorithm to other existing manifold alignment procedures.
Many existing algorithms begin with a set of high-dimensional points sampled or observed from manifolds in $\mathbb{R}^k$; see, for example, \cite{ham2005semisupervised,MMCV4,sharma2012generalized}.  
Dimension reduction techniques are then applied jointly to the observations to align the manifolds in a common $d$-dimensional embedding space with $d\ll k$.
In JOFC---similar to many of MDS and kernel based methods; see, for example, \cite{leeuw2008multidimensional, MMML1,shen2014manifold}---often the objects' measurements cannot be made in Euclidean space.  For example, the views of a single object may represent the i.\@ text content, ii.\@ images, iii.\@ communication activity associated with a single social media profile.
While these data are non-Euclidean by nature, nonetheless there are well established dissimilarities that can be computed within each modality.
Indeed, the only requirement in the JOFC framework is that we can compute dissimilarities amongst the data points} within \emph{each modality.
}
\end{remark}

\subsection{Three-Way Raw Stress MDS }
\label{S:3SMDS}
In both the 3-RSMDS and the JOFC frameworks, we seek to simultaneously embed the $m$ object-wise dissimilarity matrices, and in both regimes, the $m$ dissimilarities are measured between the same $n$ objects; i.e., they are produced by repeated measurements or observations under potentially disparate modalities.
Assuming that the entire cross-modality correspondence is known {\it a priori} between the $n$ objects, {\it Three-Way Raw Stress Multidimensional Scaling} (3-RSMDS) seeks to find a configuration $${\bf X}^\top=\left[
({\bf X}^{(1)})^\top| 
({\bf X}^{(2)})^\top|
\cdots|
({\bf X}^{(m)})^\top\right]\in\mathbb{R}^{mn\times d},
$$ of the $mn$ points that minimizes the raw stress criterion,
\begin{equation}
\label{eq:3waystress}
\tilde\sigma(\bX)=\sum_{i=1}^m \sum_{j<k}\left([\Delta_i]_{j,k}-d_{j,k}({\bf X}^{(i)})\right)^2,
\end{equation}
subject to the constraint that ${\bf X}^{(i)}={\bf GW}^{(i)}$ for all $i\in[m]:=\{1,2,\ldots,m\}$ (note that to remove nonidentifiability issues, $\bG$ is often constrained to satisfy $\bG \bG^\top=I_{n}$). 
In (\ref{eq:3waystress}), for ${\bf M}\in\mathbb{R}^{k\times \ell}$, $d_{i,j}({\bf M})$ is the Euclidean distance between the $i$-th and $j$-th rows of ${\bf M}$, 
and for $i\in[n]$, $\bX^{(i)}$ are the embedded points in $\mathbb{R}^d$ corresponding to $\Delta_i$.
Adopting the terminology in \cite{borg2005modern}, in the dimension-weighting 3-RSMDS model, ${\bf G}$ is known as the {\it group stimulus space}, and the ${\bf W}^{(i)}$ are diagonal matrices with nonnegative diagonal entries.  
In this model, the individual embeddings ${\bf X}^{(i)}$ differ only in the (potentially different) weights--- given by the diagonal entries of the respective ${\bf W}^{(i)}$'s---they place on the dimensions of ${\bf G}$.  

The 3-RSMDS dimension weighting model and its variants have been well-studied in the literature; see, for example, \cite{indscal,idioscal,idio3,de1980multidimensional,heiser1988proxscal,harshman1984parafac}.  
Indeed, there are a number of proposed procedures in the literature for solving the Three-Way MDS problem under a variety of error criterion, including the INDSCAL algorithm of \cite{indscal}; the IDIOSCAL algorithm of \cite{idioscal,idio3}; the PROXSCAL algorithm of \cite{heiser1988proxscal}; and the PARAFAC algorithm of \cite{harshman1984parafac}; among numerous others.
We note here that minimizing (\ref{eq:3waystress}) subject to the constraint ${\bf X}^{(i)}={\bf GW}^{(i)}$ for all $i\in\{1,2,\ldots,m\}$ is equivalent to performing constrained Raw Stress MDS on the dissimilarity matrix
\begin{equation}
\label{eq:rsdelta}
\widetilde{\bf \Delta}=\begin{bmatrix}
\D_1&\text{NA}&\cdots &\text{NA} \\
\text{NA}&\D_2&\cdots &\text{NA} \\
\vdots & \vdots & \ddots & \vdots\\
\text{NA}&\text{NA}&\cdots &\D_m
\end{bmatrix}\in\mathbb{R}^{mn\times mn}
\end{equation}
 with configuration matrix 
$${\bf X}^\top=\left[
({\bf X}^{(1)})^\top| 
({\bf X}^{(2)})^\top|
\cdots|
({\bf X}^{(m)})^\top\right]\in\mathbb{R}^{mn\times d},
$$
subject to ${\bf X}^{(i)}={\bf GW}^{(i)}$ for all $i\in[m]$.  
The ``NA'' entries in ${\bf \Delta}$ represent the reality that the dissimilarities across modalities are unknown a priori.
This is accounted for in the objective function by zeroing out the contribution to the stress associated with these entries of $\widetilde{\bf \Delta}$. 
the weight matrix $\bf W$ is structured to zero out the missing data entries of $\eta$ in the objective function $\sigma(\cdot).$.
The constrained MDS iterative majorization algorithm of \cite{de1980multidimensional} can then be applied to approximately solve the 3-RSMDS model.  
As the JOFC procedure (see Algorithm \ref{alg:1}) and the accelerated fJOFC procedure (see Algorithm \ref{alg:2}) are both iterative majorization MDS procedures, we will provide the details of \cite{de1980multidimensional} applied to 3-RSMDS for the sake of comparison.  The procedure of \cite{de1980multidimensional} consists of the following two iterated steps, given an initialization of the configuration $\bX_{(0)}$:
\begin{itemize}
\item[1.]  At configuration $\bX_{(t-1)}$, ignoring the constraint that ${\bf X}^{(i)}={\bf GW}^{(i)}$ for all $i\in[m]$, compute the unconstrained update $\widetilde{\bf X}_{(t)}$ via the Guttman transform; see \cite{borg2005modern}.
\item[2.]  Set 
$\bX_{(t)}=\left[
\big({\bf X}^{(1)}_{(t)}\big)^\top| 
\big({\bf X}^{(2)}_{(t)}\big)^\top|
\cdots|
\big({\bf X}^{(m)}_{(t)}\big)^\top\right]$ to be the minimizer of
$$\tr(\bX-\widetilde\bX_{(t)})^{\top}\widetilde{\bf L}(\bX-\widetilde\bX_{(t)}),$$
over $\bX$ subject to the constraints ${\bf X}^{(i)}={\bf GW}^{(i)}$ for all $i\in\{1,2,\ldots,m\}$.  Here, $\widetilde{\bf L}\in\mathbb{R}^{mn\times mn}$ is the block diagonal matrix with $nI_n-J_n\in\mathbb{R}^{n\times n}$ in each of the $m$ diagonal blocks, where $J_n={\bf 1}_n{\bf 1}_n^T\in\mathbb{R}^{n\times n}$, and ${\bf 1}_n$ is the column vector of all one's in $\mathbb{R}^n$.  
This minimization is often approached by alternating minimizing over $\bG$ for a fixed $\bW$ and then minimizing over $\bW$ for a fixed $\bG$.
\end{itemize}

\subsubsection{The JOFC framework}
\label{S:JOFC}
In the above 3-RSMDS framework, the matchedness of the $n$ observations across the $m$ dissimilarities is enforced via the ${\bf X}^{(i)}={\bf GW}^{(i)}$ constraints.
In the JOFC algorithm, the matchedness constraint is built into the objective function as follows.
Contrasting the raw stress criterion in (\ref{eq:3waystress}), the variant of JOFC we consider seeks to produce an unconstrained configuration ${\bf X}^\top=\left[
({\bf X}^{(1)})^\top| 
({\bf X}^{(2)})^\top|
\cdots|
({\bf X}^{(m)})^\top\right]\in\mathbb{R}^{mn\times d},$
(where $({\bf X}^{(i)})^\top=\left[
({ X}^{(i)}_1)^\top| 
({ X}^{(i)}_2)^\top|
\cdots|
({X}_m^{(i)})^\top\right]\in\mathbb{R}^{n\times d},
$ are the points associated with $\Delta_i$)
that minimizes the raw stress criterion
\begin{align}
\label{stress}
\sigma({\bf X})=\underbrace{\sum_{i=1}^m \sum_{1\leq j <\, \ell\leq n}\left([\Delta_i]_{j,\ell}-d_{j,\ell}({\bf X}^{(i)})\right)^2}_{\text{fidelity}}  + w\underbrace{\sum_{1\leq i < j\leq m}\sum_{\ell=1}^n d(X^{(i)}_\ell,X^{(j)}_\ell)^2}_{\text{commensurability}},
\end{align}
where $d(\cdot,\cdot)$ is the Euclidean distance function. 
The raw stress criterion in JOFC is composed of three major pieces: 
\begin{itemize}
\item[1.] The ``fidelity'' term, $\sum_{i=1}^m \sum_{1\leq j <\, \ell\leq n}\left([\Delta_i]_{j,\ell}-d_{j,\ell}({\bf X}^{(i)})\right)^2$, which measures the faithfulness of the embedding to the original dissimilarities, $\{\Delta_i\}_{i=1}^m$.
Note the the fidelity is equal to the raw stress criterion in 3-RSMDS (\ref{eq:3waystress}).
\item[2.] The ``commensurability'' term, $\sum_{1\leq i <j\leq m}\sum_{\ell=1}^n d(X^{(i)}_\ell,X^{(j)}_\ell)^2$,
which measures how the geometry of the embeddings differs across modality.  Similar to the role of the ${\bf X}^{(i)}={\bf GW}^{(i)}$ constraints in 3-RSMDS, in JOFC the commensurability term (softly) enforces the matchedness of the $n$ data points across the $m$ modalities.
We also note that the commensurability is proportional to the objective function of three-way Procrustes analysis
\begin{align}
\label{eq:comm}
\text{commensurability}&=\sum_{i<j}^m \tr({\bf X}^{(i)}-{\bf X}^{(j)})^\top({\bf X}^{(i)}-{\bf X}^{(j)})\notag\\
&=m\sum_{i=1}^m \tr({\bf X}^{(i)}-\bar{\bf X})^\top({\bf X}^{(i)}-\bar{\bf X}),
\end{align}
where $\bar{\bf X}=m^{-1}\sum_{i=1}^m \bX^{(i)}.$
\item[3.] The weighting of the fidelity versus the commensurability of the embedding provided by $w$.  
If $w\ll 1$, then the optimal embedding will preserve the within-modality dissimilarities at the expense of the cross-modality correspondence; i.e. each $\Delta_i$ will be fit separately.  If $w\gg 1$, then the optimal embedding will preserve the cross-modality correspondence at the expense of the within-modality dissimilarities; i.e. from Eq (\ref{eq:comm}) we see that $w\gg 1$ would force all of the $\bX^{(i)}$ to be equal without concern for preserving the original $\Delta_i$'s.  
In light of this, JOFC can be viewed as weakly constrained Raw Stress MDS (see \cite{borg2005modern} for detail), with $w$ allowing us to continuously range between setting all $\bX^{(i)}$'s to be equal but otherwise unconstrained ($w=\infty$) at one extreme versus embedding the $\Delta_i$'s completely separately ($w=0$) at the other.

The problem of choosing an optimal $w$ was taken up in \cite{adali2013fidelity}.
When the individual dissimilarities are normalized to have have $\|\D_i\|_F=1$ for all $i\in[m]$, the results of \cite{adali2013fidelity} suggest that, under suitable model assumptions, the 
performance of the JOFC
procedure is relatively robust to the choice of $w$.
In application, a data-adaptive $w$ could be chosen via the bootstrapping AUC-optimization testing procedure of \cite{adali2013fidelity}, although we do not pursue this further here.
\end{itemize}

As in 3-RSMDS, minimizing (\ref{stress}) can be seen as unconstrained Raw Stress MDS on the omnibus dissimilarity matrix 
\begin{equation}
\label{eq:jofcdelta}
{\bf \Delta}=[{\bf\Delta}_{i,j}]=\begin{bmatrix}
\D_1&\eta&\cdots &\eta \\
\eta&\D_2&\cdots &\eta \\
\vdots & \vdots& \ddots & \vdots\\
\eta&\eta&\cdots &\D_m
\end{bmatrix}\in\mathbb{R}^{mn\times mn}, \hspace{5mm}\eta=\begin{bmatrix}
0&\text{NA}&\cdots &\text{NA} \\
\text{NA}&0&\cdots &\text{NA} \\
\vdots & \vdots &\ddots & \vdots\\
\text{NA}&\text{NA}&\cdots &0
\end{bmatrix}\in\mathbb{R}^{n\times n},
\end{equation}
and configuration 
$${\bf X}^\top=\left[
({\bf X}^{(1)})^\top| 
({\bf X}^{(2)})^\top|
\cdots|
({\bf X}^{(m)})^\top\right]\in\mathbb{R}^{mn\times d},
$$ with the associated weight matrix given by 
\begin{equation}
\label{eq:W}
{\bf W}=[W_{i,j}]=\begin{bmatrix}
J_n-I_n&wI_n&\cdots &wI_n \\
wI_n&J_n-I_n&\cdots &wI_n \\
\vdots & \vdots & \ddots & \vdots\\
wI_n&wI_n\cdots&\cdots &J_n-I_n
\end{bmatrix}\in\mathbb{R}^{mn\times mn};
\end{equation}
indeed, this is immediate as the raw stress criterion in (\ref{stress}) is equal to $\sigma({\bf X})=\sum_{i<j} W_{i,j}(\Delta_{i,j}-d_{i,j}({\bf X}))^2.$  
Note that, as before, the weight matrix $\bf W$ is structured to zero out the missing data entries of $\eta$ in the objective function $\sigma(\cdot).$

Note the different structure of ${\bf \Delta}$ in JOFC versus $\widetilde{\bf \Delta}$ in 3-RSMDS.  In JOFC, we impute the missing across modality dissimilarity between the same object to be 0, which allows us to build the matchedness constraint into the raw stress criterion (via the commensurability term).
In both models, we treat inter-object, cross-modality dissimilarites as missing data, and this represents the assumption that these dissimilarites are often {\it not} available
in the embedding procedure.
\begin{remark}
\label{rem:oldJOFC}
\emph{In \cite{JOFC}, the missing cross-modality dissimilarity between modality $i$ and modality $j$ was imputed as $(\D_i+\D_j)/2$, and ${\bf \D}$ was embedded using classical multidimensional scaling.  
Here we choose not to impute the missing data for two main reasons:  imputing the cross-modality dissimilarities potentially increases the variance in our embedded points; and the special structure of $\bf W$ in the missing data setting allows us to greatly speed up and parallelize the JOFC procedure (see Section \ref{sec:fastJOFC}).
In addition, in many real data settings (see Section \ref{S:results}) the $n$ objects originate from {\it disparate} data sources and are not simply repeated measurements of the same objects in a single space, which further complicates the very concept of cross-modality dissimilarities.} 
\end{remark} 

Similar to the approach in \cite{de1980multidimensional} for 3-RSMDS, our JOFC approach embeds ${\bf \D}$ by minimizing (\ref{stress}) via successive Guttman transforms.  
As in the majorization algorithm for solving 3-RSMDS, the Guttman transform step of JOFC can be efficiently computed (see Algorithm \ref{alg:2}).
However, in JOFC the matchedness constraint is built into the raw stress criterion, and we are therefore able to avoid the potentially costly Step 2 of the 3-RSMDS procedure as outlined in Section \ref{S:3SMDS}.
The JOFC algorithm proceeds as follows:
\begin{itemize}
\item[1.]  Initialize the configuration ${\bf X}_{(0)}$. 
One easily implemented initialization imputes the missing data entries of ${\bf \Delta}$ as in Remark \ref{rem:oldJOFC} and performs classical MDS on ${\bf \Delta}$; see Step 1 of Algorithm \ref{alg:1} for detail.
\item[2.] For a given threshold $\epsilon>0$, while $\sigma({\bf X}_{(t)})-\sigma({\bf X}_{(t-1)})>\epsilon$, iteratively update $\bX_{t-1}$ via the Guttman transform.
To wit, let ${\bf L}$ be the combinatorial Laplacian of the weight matrix ${\bf W}$ (i.e., if ${\bf D}$ is the diagonal matrix with $D_{i,i}=\sum_{j} W_{i,j}$, then ${\bf L=D-W}$), and define 
\begin{equation}
\label{eq:B}
B ({\bf X})_{i,j}:=\begin{cases}
\frac{-W_{i,j}{\bf\Delta}_{i,j}}{d_{i,j}({\bf X})}&\text{ if }i\neq j\text{ and }d_{i,j}({\bf X})\neq 0\\
0 &\text{ if }i\neq j\text{ and }d_{i,j}({\bf X})= 0\\
-\sum_{k=1,k\neq i}^n B({\bf X})_{i,j}&\text{ if }i= j.
\end{cases}\end{equation}
Then the raw stress criterion (\ref{stress}) can be written
\begin{align*}
\sigma(\bX_{(t)})=\sum_{i<j}W_{i,j}{\bf\Delta}_{i,j}^2+\tr\bX_{(t)}^\top \bl\bX_{(t)}-2\tr \bX_{(t)}^\top B(\bX_{(t)}) \bX_{(t)},
\end{align*}
which is majorized by
\begin{align}
\label{eq:major}
\sigma(\bX_{(t)})\leq\sum_{i<j}W_{i,j}{\bf\Delta}_{i,j}^2+\tr\bX_{(t)}^\top \bl\bX_{(t)}-2\tr \bX_{(t)}^\top B(\bX_{(t-1)}) \bX_{(t-1)},
\end{align}
a quadratic function of $\bX_{(t)}$.  
The minimizer of (\ref{eq:major}) can be found by solving the stationary equation $\nabla \sigma({\bf X}_{(t)}) =2{\bf L}{\bf X}_{(t)} - 2B(\bX_{(t-1)})\bX_{(t-1)}= 0$. 
The Guttman transform updates a configuration $\bX_{(t-1)}$ by solving $\bl \bX_{(t)} = B(\bX_{(t-1)})\bX_{(t-1)}$; in the multidimensional scaling literature, this transformation is often written as $\bX_{(t)} = \Gamma (\bX_{(t-1)}) = \bl^\dagger B(\bX_{(t-1)})\bX_{(t-1)}$ where $\bl^\dagger$ is the Moore-Penrose pseudoinverse of $\bl$.  
Notice that $\bX_{(t)}$ is centered at zero even if $\bX_{(t-1)}$ is not centered at zero.  
  \end{itemize}
\begin{algorithm}[t!]
  \caption*{SMACOF algorithm for raw stress multidimensional scaling} 
\begin{algorithmic}[1]
\caption{JOFC Algorithm for Manifold Matching (see Section \ref{S:JOFC} for detail)} \label{alg:1}
  \REQUIRE Omnibus dissimilarity matrix ${\bf \Delta}$, weight matrix ${\bf W}$, embedding dimension $d$, tol$=\epsilon$
  \ENSURE ${\bf X}\in\mathbb{R}^{mn\times d}$, a configuration of points in $\mathbb{R}^d$
  \STATE Initialize ${\bf X}_{(0)}$ via cMDS (classical MDS , see \cite{torgerson1952multidimensional,borg2005modern} for detail) of ${\bf \Delta}$  \label{step:init} \\
  i.  Set ${\bf \Delta}^{(2)}$ to be the element-wise square of ${\bf \Delta}$; i.e., ${\bf \Delta}^{(2)}_{i,j}=({\bf \Delta}_{i,j})^2$;\\
  ii.  Compute ${\bf P}=-\frac{1}{2}(I_{mn}-\frac{1}{mn}J_{mn}){\bf \Delta}^{(2)}(I_{mn}-\frac{1}{mn}J_{mn})$;\\
  iii.  Compute the $d$ largest eigenvalues $\lambda_1,\lambda_2,\cdots,\lambda_d$ of $\bf P$ with corresponding eigenvectors $u_1,u_2,\ldots,u_d$;\\
  iv.  Set ${\bf X}_{(0)}=[u_1|u_2|\cdots|u_d]\text{diag}(\lambda_i)^{1/2}$;
  \STATE Compute $\sigma({\bf X}_{(0)})$
  \WHILE{$\sigma({\bf X}_{(t)})-\sigma({\bf X}_{(t-1)})>\epsilon$}
  \STATE ${\bf X}_{(t)}={\bf L}^{\dagger}B({\bf X}_{(t-1)}){\bf X}_{(t-1)}$ 
  \STATE Compute $\sigma({\bf X}_{(t)})$
  \ENDWHILE
  \STATE Output the final iteration ${\bf X}_{\text{(final)}}$
\end{algorithmic}
\end{algorithm}
For JOFC, the resulting iterative algorithm is summarized in Algorithm \ref{alg:1}.  
Note that the sequence of steps generated by successive Guttman transforms is derived via majorization, and we note that Algorithm \ref{alg:1} is closely related to the popular SMACOF algorithm for metric multidimensional scaling; see \cite{de1980multidimensional,smacof2}.

\begin{remark}
\label{rem:eps}
\emph{ In all of the experiments in Section \ref{S:results}, the threshold $\epsilon$ is set to $10^{-6}\binom{nm}{2};$ i.e., we terminate the procedure when the normalized stress $\sigma_N(\cdot):=\sigma(\cdot)/\binom{nm}{2}$ fails to decrease by at least $10^{-6}$ between successive iterations.
Note however that, in practice, the sequential Guttman transforms often exhibit good global properties, and only a few iterations are required to obtain a sufficiently good suboptimal embedding, see \cite{kearsley1995solution}.  We empirically observe this phenomena in Figure \ref{fig:mmee}, where we see that the configuration obtained by fJOFC can stabilize after only relatively few iterates.
}
\end{remark}

In general, $\bl^\dagger$ must be calculated by singular value or QR decomposition, which may be prohibitively expensive if $mn$ is large, with computational complexity of order $O(m^3n^3)$.  
Fortunately, there are many applications in which the special structure of the weight matrix ${\bf W}$ allows for direct calculation of $\bl^\dagger$, sometimes with subsequent simplification of $\bl^\dagger B(\bX_{(t-1)})\bX_{(t-1)}$.  
Examples include the familiar case of unit weights (which is the case for the Guttman transform needed in Step 1 of the 3-RSMDS algorithm in Section \ref{S:3SMDS}) and the case of symmetric block-circulant matrices, see \cite{gower1990applications,gower2006application}. 
In Section \ref{sec:fastJOFC}, we demonstrate that the special structure of JOFC also permits the direct calculation of $\bl^\dagger$ which then results in a much simplified calculation of $\bl^\dagger B(\bX_{(t-1)})\bX_{(t-1)}$.

\section{Fast JOFC}
\label{sec:fastJOFC}
In each iteration of the JOFC algorithm (Algorithm \ref{alg:1}), we update the configuration via a Guttman transform ${\bf X}_{(t)}={\bf L}^{\dagger}B({\bf X}_{(t-1)}){\bf X}_{(t-1)}$.  
Computationally, this involves 
\begin{itemize}
\item[1.] A single calculation of ${\bf L}^\dagger$, which naively has algorithmic complexity $O((mn)^3)$ given an SVD (or QR decomposition) based pseudoinverse algorithm.  Clearly, as ${\bf L}^\dagger$ does not vary in $t$, we do not need to recalculate this pseudoinverse in every iteration.
\item[2.] Computing ${\bf L}^{\dagger}B({\bf X}_{(t-1)}){\bf X}_{(t-1)}$, which has complexity $O((mn)^2d).$   
\end{itemize}
Therefore, given a bounded number of iterations and assuming $d<mn$, the JOFC algorithm has algorithmic complexity $O((mn)^3)$.

To speed up the JOFC procedure, we first note that the form of the JOFC weight matrix allows us to algebraically compute ${\bf L}^\dagger$.  Next, we show that the resulting form of the pseudoinverse allows us to greatly simplify the computation of ${\bf L}^{\dagger}B({\bf X}_{(t-1)}){\bf X}_{(t-1)}$.
In addition, the computation of ${\bf L}^{\dagger}B({\bf X}_{(t-1)}){\bf X}_{(t-1)}$ easily lends itself to parallelization.

\subsection{Computing ${\bf L}^\dagger$}
\label{S:Ldagger}

The first step in speeding up Algorithm \ref{alg:1} is algebraically computing the pseudoinverse ${\bf L}^\dagger$.
Here, we present the computation of $\bf L^\dagger$ in the case of a more general weight matrix than considered in Eq. (\ref{eq:W}); namely, we will consider here $\bW$ of the form 
\begin{equation}
\label{eq:W2}
{\bf W}=[W_{i,j}]=\begin{bmatrix}
w_{1,1}(J_n-I_n)&w_{1,2}I_n&\cdots &w_{1,m}I_n \\
w_{2,1}I_n&w_{2,2}(J_n-I_n)&\cdots &w_{2,m}I_n \\
\vdots & \vdots & \ddots & \vdots\\
w_{m,1}I_n&w_{m,2}I_n\cdots&\cdots &w_{m,m}(J_n-I_n)
\end{bmatrix}\in(\mathbb{R}^+)^{mn\times mn};
\end{equation}
with $w_{i,j}=w_{j,i}$ for all $i,j\in[m]$ such that $i\neq j$.
This form of $\bW$ allows for different weightings across and within modalities.
The case of equal weights off diagonal, i.e., the $\bW$ in Eq. (\ref{eq:W}), will then be realized as a special case of this more general $\bW$.

In Appendix \ref{app:1}, we prove the following.
Writing
\begin{equation}
\label{eq:modW}
\mathcal{W}\!=\!\!\begin{bmatrix}
nw_{1,1}+\sum_{j\neq 1}w_{1,j}&-w_{1,2}&\cdots &-w_{1,m} \\
-w_{2,1}&nw_{2,2}+\sum_{j\neq 2}w_{2,j}&\cdots &-w_{2,m} \\
\vdots & \vdots & \ddots & \vdots\\
-w_{m,1}&-w_{m,2}\cdots&\cdots &nw_{m,m}+\sum_{j\neq m}w_{m,j}
\end{bmatrix},
\end{equation}
and
$\text{diag}(w_{i,i}):=\text{diag}(w_{1,1},w_{2,2},\cdots,w_{m,m})$,
we algebraically compute $\bf L^\dagger$ via 
$$\bl^\dagger=\mathcal{W}^{-1}\otimes I_n+\left[-\left(\mathcal{W}+n\left(\frac{J_m}{mn}-\text{diag}(w_{i,i})\right)\right)^{-1}\left(\frac{J_m}{mn}-\text{diag}(w_{i,i})\right)\mathcal{W}^{-1}-\frac{J_m}{mn} \right]\otimes J_n.$$
While brute force computation of $\mathcal{W}^{-1}$ (and $\mathcal Z$) would incur a $O(m^3)$ cost as opposed to the $O(m^3n^3)$ cost of a brute force computation of $\bl^\dagger$, structured weight matrices can greatly simplify this computation.  For example,
if ${\bf W}$ is of the form of Eq. (\ref{eq:W}), then 
a brief calculation yields that
\begin{equation}
\label{eq:W1inv}
{\mathcal W}^{-1}=\begin{bmatrix}
\frac{n+w}{n(n+mw)}&\frac{w}{n(n+mw)}&\cdots &\frac{w}{n(n+mw)}\\
\frac{w}{n(n+mw)}&\frac{n+w}{n(n+mw)}&\cdots &\frac{w}{n(n+mw)}\\
\vdots & \vdots & \ddots & \vdots\\
\frac{w}{n(n+mw)}&\frac{w}{n(n+mw)} \cdots&\cdots &\frac{n+w}{n(n+mw)}
\end{bmatrix}\in\mathbb{R}^{m\times m},\end{equation}
and
\begin{align*}
&-\left(\mathcal{W}+n\left(\frac{J_m}{mn}-\text{diag}(w_{i,i})\right)\right)^{-1}\left(\frac{J_m}{mn}-\text{diag}(w_{i,i})\right)\mathcal{W}^{-1}-\frac{J_m}{mn} \\
&\hspace{10mm}=\begin{bmatrix}
\frac{-m^2 w^2+m n^2-m n w-n^2}{wn^2m^2(n+wm)}&\frac{ -m^2 w^2-m n w-n^2}{wn^2m^2(n+wm)}&\cdots &\frac{ -m^2 w^2-m n w-n^2}{wn^2m^2(n+wm)}\\
\frac{ -m^2 w^2-m n w-n^2}{wn^2m^2(n+wm)}&\frac{-m^2 w^2+m n^2-m n w-n^2}{wn^2m^2(n+wm)}&\cdots &\frac{ -m^2 w^2-m n w-n^2}{wn^2m^2(n+wm)}\\
\vdots & \vdots & \ddots & \vdots\\
\frac{ -m^2 w^2-m n w-n^2}{wn^2m^2(n+wm)}&\frac{ -m^2 w^2-m n w-n^2}{wn^2m^2(n+wm)} \cdots&\cdots &\frac{-m^2 w^2+m n^2-m n w-n^2}{wn^2m^2(n+wm)}
\end{bmatrix}\in\mathbb{R}^{m\times m}.
\end{align*}
We shall see in Section \ref{S:speedupGutman} how these algebraic computations greatly speed-up the computation of the Guttman transform in the fJOFC procedure.

Also note that in implementing the fJOFC algorithm, only $\mathcal{W}^{-1}$ needs to be computed.
Indeed,  
${\bf 1}_{mn}^\top B({\bf X}_{(t-1)})=B({\bf X}_{(t-1)}){\bf 1}_{mn}=0,$
which immediately implies that
\begin{align*}
\left(\left[\!-\!\left(\mathcal{W}+n\!\!\left(\frac{J_m}{mn}-\text{diag}(w_{i,i})\right)\right)^{-1}\!\!\!\left(\frac{J_m}{mn}-\text{diag}(w_{i,i})\right)\mathcal{W}^{-1}\!-\!\frac{J_m}{mn} \right]\otimes J_n\right) &B({\bf X}_{(t-1)})={\bf 0}_{mn}.
\end{align*}
Resultingly, the Gutman transform in the $t$-th iteration of fJOFC is computed simply as 
$${\bf X}_{(t)}=\left(\mathcal{W}^{-1}\otimes I_n\right) B({\bf X}_{(t-1)}){\bf X}_{(t-1)}.$$

\begin{remark}
\emph{
  The key to computing the form of $\bf L^\dagger$ is realizing that $\bf L^\dagger$ can be written as 
  $${\bf L}^{\dagger}=\left({\bf L}+\frac{1}{mn}J_{mn}\right)^{-1}-\frac{1}{mn}J_{mn}.$$
  We then compute the exact form of $\left({\bf L}+\frac{1}{mn}J_{mn}\right)^{-1}$ by inverting the structured matrix
  $${\bf L}+\frac{1}{mn}J_{mn}
=\mathcal W\otimes I_n+\left(\frac{1}{mn}J_{m}-\text{diag}(w_{i,i})\right)\otimes J_n.$$
This inverse computation (Theorem \ref{thm:1} in Appendix \ref{app:1}) can be generalized to the following Woodbury-type \citep{woodbury1950inverting} matrix identity for the sum of Kronecker products.  Let $A,B\in\mathbb{R}^{m\times m}$ be matrices such that  $A$ and $(A+nB)$ are invertible matrices.  Then it follows that 
  $$(A\otimes I_n+B\otimes J_n)^{-1}= A^{-1}\otimes I_n-(A+nB)^{-1}BA^{-1}\otimes J_n.$$
 This formula generalizes Theorem \ref{thm:1}, and we are presently exploring different use cases for such an identity.
}
\end{remark}
\begin{remark}
\emph{
Even given identical initializations, the fJOFC algorithm (Algorithm 2), and the JOFC algorithm 
may not give identical embeddings of ${\bf \Delta}$, as JOFC relies on a computational approximation of ${\bf L}^\dagger$, while fJOFC exactly algebraically computes ${\bf L}^\dagger.$
}
\end{remark}

\subsubsection{More general weight matrices}
\label{rem:otherW's}
We described above how the structured $\bW$ of Eq. (\ref{eq:W}) offers an easily computed form for $\mathcal{W}^{-1}$, and here we will briefly outline some other potentially useful structured weight matrices that lend themselves to easily compute $\mathcal{W}^{-1}$.  
If $\bW$ is of the form
\begin{equation}
\label{eq:W3}
{\bf W}=[W_{i,j}]=\begin{bmatrix}
w_{1,1}(J_n-I_n)&w_{1,1}w_{2,2}I_n&\cdots &w_{1,1}w_{m,m}I_n \\
w_{1,1}w_{2,2}I_n&w_{2,2}(J_n-I_n)&\cdots &w_{2,2}w_{m,m}I_n \\
\vdots & \vdots & \ddots & \vdots\\
w_{1,1}w_{m,m}I_n&w_{2,2}w_{m,m}I_n\cdots&\cdots &w_{m,m}(J_n-I_n)
\end{bmatrix}\in(\mathbb{R}^+)^{mn\times mn};
\end{equation}
so that each modality has its own (potentially unique) weight, and the cross modality dissimilarities are weighted via a product of the within modality weights, 
then 
$$\mathcal{W}^{-1}=\frac{\text{diag}(w_{i,i})^{-1}}{n+\sum_i w_{i,i}}+\frac{1}{n(n+\sum_i w_{i,i})}J_m,$$
so that the $k,\ell$-th entry of $\mathcal{W}^{-1}$ is equal to
$$\mathcal{W}^{-1}_{k,\ell}=\begin{cases}
\frac{1}{w_{k,k}(n+\sum_i w_{i,i})}+\frac{1}{n(n+\sum_i w_{i,i})}&\text{ if } k=\ell\\
 \frac{1}{n(n+\sum_i w_{i,i})}&\text{ else. }\end{cases}$$
 Increasing the weight of the within-modality embeddings can easily be achieved by letting $\cal W$ be set to
 $${\cal W}=c n\cdot\text{diag}(w_{i,i})+\text{diag}(w_{i,i})\left( \left(\sum_iw_{i,i}\right)I_m -J_m \text{diag}(w_{i,i})\right),$$
 in which case
 $$\mathcal{W}^{-1}=\frac{\text{diag}(w_{i,i})^{-1}}{cn+\sum_i w_{i,i}}+\frac{1}{cn(cn+\sum_i w_{i,i})}J_m.$$
 Increasing (resp., decreasing) the value of the constant $c$ will have the effect of emphasizing (resp., deemphasizing) the fidelity of the subsequent embedding.

\subsection{Effect on the computation of ${\bf L}^{\dagger}B({\bf X}_{(t-1)}){\bf X}_{(t-1)}$}
\label{S:speedupGutman}
Exploiting the form of $\bl^\dagger$ computed above, we use the special structure of $B({\bf X}_{(t-1)})$ to simplify and speed up the calculation of the Guttman transform needed in the $t$-th iteration of the JOFC algorithm.  

We first note that $B({\bf X}_{(t-1)})$ is block diagonal, with $m$ diagonal blocks each of size $n\times n.$  
We will denote the diagonal blocks of $B({\bf X}_{(t-1)})$ by $B_1,B_2,\ldots,B_m.$
By construction, 
$${\bf 1}_{mn}^\top B({\bf X}_{(t-1)})=B({\bf X}_{(t-1)}){\bf 1}_{mn}=0,$$ and therefore ${\bf 1}_{n}^\top B_j=B_j{\bf 1}_{n}=0$ for all $j=1,2,\ldots,m.$  
It follows that $B_j J_n=J_n B_j=0$ for all $j=1,2,\ldots,m.$
Defining
$$A':=\frac{n+w}{n(n+mw)}I_n,\text{ and }C':=\frac{w}{n(n+mw)}I_n,$$
we arrive at 
\[{\bf L}^{\dagger}B({\bf X}_{(t-1)})=\begin{bmatrix}
A'&C'&\cdots &C' \\
C'&A'&\cdots &C' \\
\vdots & \vdots & \ddots & \vdots\\
C'&C'&\cdots &A'
\end{bmatrix}\begin{bmatrix}
B_1&0&\cdots &0 \\
0&B_2&\cdots &0 \\
\vdots & \vdots & \ddots & \vdots\\
0&0&\cdots &B_m
\end{bmatrix},\]
and so
\begin{align}
\label{eq:speedup}
\bX_{(t)}&={\bf L}^{\dagger}B({\bf X}_{(t-1)}){\bf X}_{(t-1)}=\begin{bmatrix}
A'&C'&\cdots &C' \\
C'&A'&\cdots &C' \\
\vdots & \vdots & \ddots & \vdots\\
C'&C'&\cdots &A'
\end{bmatrix}\begin{bmatrix}
B_1&0&\cdots &0 \\
0&B_2&\cdots &0 \\
\vdots & \vdots & \ddots & \vdots\\
0&0&\cdots &B_m
\end{bmatrix}\begin{bmatrix}
\bX_{(t-1)}^{(1)} \\
\bX_{(t-1)}^{(2)} \\
\vdots \\
\bX_{(t-1)}^{(m)}
\end{bmatrix}\notag\\
&=\left(\frac{n}{n(n+mw)}I_{nm}+\frac{w}{n(n+mw)}\begin{bmatrix}
I_n&I_n&\cdots &I_n \\
I_n&I_n&\cdots &I_n \\
\vdots & \vdots & \ddots & \vdots\\
I_n&I_n&\cdots &I_n
\end{bmatrix}\right)\begin{bmatrix}
B_1\bX_{(t-1)}^{(1)} \\
B_2\bX_{(t-1)}^{(2)} \\
\vdots \\
B_m\bX_{(t-1)}^{(m)}
\end{bmatrix}.
\end{align}
From (\ref{eq:speedup}), it is immediate that the update is realized via
\begin{equation}
\label{eq:update}
\bX_{(t)}^{(j)}= \frac{n}{n(n+mw)}B_j\bX_{(t-1)}^{(j)}+\sum_{\ell=1}^m \frac{w}{n(n+mw)}B_\ell\bX_{(t-1)}^{(\ell)}.
\end{equation}
\begin{remark}
\label{rem:pargut}
\emph{Note that to efficiently compute (\ref{eq:update}), we can first compute each $B_\ell\bX_{(t-1)}^{(\ell)}$ in parallel 
for $\ell\in[m],$ and then compute the update in Eq. (\ref{eq:update}).}
\end{remark}

\subsection{The fJOFC algorithm}
The algebraic computation of $\bl^\dagger$ in Section \ref{S:Ldagger} combined with the computation of the Guttman transform of Section \ref{S:speedupGutman} combine to give us the fJOFC algorithm, which is detailed below and in Algorithm \ref{alg:2}.
\begin{algorithm}[t!]
\begin{algorithmic}[1]
\caption{fJOFC: Fast JOFC Algorithm for Manifold Matching} \label{alg:2}
  \REQUIRE Omnibus dissimilarity matrix ${\bf \Delta}$, weight matrix ${\bf W}$, embedding dimension $d$, tol$=\epsilon$
  \ENSURE ${\bf X}\in\mathbb{R}^{mn\times d}$, a configuration of points in $\mathbb{R}^d$
  \STATE Set $\xi_0$ to be the configuration obtained via cMDS of $\left(\sum_i \Delta_i\right)/m$ (see Step 1. of Algorithm \ref{alg:1} for detail); Center $\xi_0$ via $\xi_0=\xi_0(I_{n}-\frac{1}{n}J_{n})$;
  \FOR{i=1,2,\ldots,m} 
  \STATE Set $\xi_i$ to be the configuration obtained via cMDS of $\Delta_i$; ; Center $\xi_i$ via $\xi_i=\xi_i(I_{n}-\frac{1}{n}J_{n})$;
  \STATE Set ${\bf X}_{(0)}^{(t)}$ to be the orthogonal Procrustes fit of $\xi_i$ onto $\xi_0$;\\
  i.  Set $T=\xi_0^T\xi_i$;\\
  ii.  Let $U\Sigma V^T$ be the singular value decomposition of $T$;\\
  iii. Set ${\bf X}_{(0)}^{(i)}=\xi_i UV^T$ 
  \ENDFOR
  \STATE Set ${\bf X}_{(0)}^\top=\left[({\bf X}_{(0)}^{(1)})^\top|({\bf X}_{(0)}^{(2)})^\top|\cdots|({\bf X}_{(0)}^{(m)})^\top\right]$
  \STATE Compute $\sigma({\bf X}_{(0)})$ as in Remark \ref{rem:faststress}
  \WHILE{$\sigma({\bf X}_{(t)})-\sigma({\bf X}_{(t-1)})>\epsilon$}
  \FOR{j=1,2,\ldots,m}
  \STATE Compute $B({\bf X}_{(t-1)}^{(j)}){\bf X}_{(t-1)}^{(j)}$
  \ENDFOR
  \FOR{j=1,2,\ldots,m}
  \STATE Set ${\bf X}_{(t)}^{(j)}=\frac{n}{n(n+nw)}B({\bf X}_{(t-1)}^{(j)}){\bf X}_{(t-1)}^{(j)}+\sum_{\ell=1}^m\frac{w}{n(n+nw)}B({\bf X}_{(t-1)}^{(\ell)}){\bf X}_{(t-1)}^{(\ell)}$
  \ENDFOR
  \STATE Set ${\bf X}_{(t)}^\top=\left[({\bf X}_{(t)}^{(1)})^\top|({\bf X}_{(t)}^{(2)})^\top|\cdots|({\bf X}_{(t)}^{(m)})^\top\right]$
  \STATE Compute $\sigma({\bf X}_{(t)})$ as in Remark \ref{rem:faststress}
  \ENDWHILE
  \STATE Output the final iteration ${\bf X}_{\text{(final)}}$
\end{algorithmic}
\end{algorithm}

The fJOFC algorithm proceeds as follows:
\begin{itemize}
\item[1.] Initialize the configuration $\bX_{(0)}$.  
If the initialization of the JOFC procedure in Remark \ref{rem:oldJOFC} is too computationally intensive (in particular, the initialization uses cMDS to embed the $mn\times mn$ omnibus dissimilarity with off-diagonal blocks imputed to be $(\Delta_i+\Delta_j)/2$) we could proceed as follows:  first, use cMDS to embed the average dissimilarity matrix $\left(\sum_i \Delta_i\right)/m$, obtaining the configuration $\xi_0$; use cMDS to embed each $\Delta_i$ and set $\bX_{(0)}^{(i)}$ to be the orthogonal Procrustes fit of the embedding to $\xi_0$---see Step 4 of Algorithm \ref{alg:2} for detail.

\item[2.]  Given current configuration $\bX_{(t-1)}$ and error threshold $\epsilon$, while $\sigma({\bf X}_{(t)})-\sigma({\bf X}_{(t-1)})>\epsilon$, compute the Guttman transform of ${\bf X}_{(t-1)}$ to obtain ${\bf X}_{(t)}$ as outlined in Section \ref{S:speedupGutman} (lines 9-15 of Algorithm \ref{alg:2}).  
To wit, first compute each $B({\bf X}_{(t-1)}^{(j)}){\bf X}_{(t-1)}^{(j)}$.  The update is then realized by setting $${\bf X}_{(t)}^{(j)}=\frac{n}{n(n+nw)}B({\bf X}_{(t-1)}^{(j)}){\bf X}_{(t-1)}^{(j)}+\sum_{\ell=1}^m\frac{w}{n(n+nw)}B({\bf X}_{(t-1)}^{(\ell)}){\bf X}_{(t-1)}^{(\ell)}$$ for all $j\in[m]$.
Each of these $m$ updates has computational complexity $O(mn^2d)$.
\end{itemize}
\begin{remark}
\label{rem:faststress}
\emph{Further speeding up the fJOFC procedure, from Eq.\@ (\ref{stress}), we see that to compute $\sigma({\bf X})$, we need not compute all $\binom{mn}{2}$ pairwise distance between rows of ${\bf X}.$  Indeed, we only need to compute $m\binom{n}{2}+\binom{m}{2}n$ interpoint distances.  Indeed, the fidelity can be written as
$$\sum_{i=1}^m \sum_{1\leq j <\, \ell\leq n}\left([\Delta_i]_{j,\ell}-d_{j,\ell}({\bf X}^{(i)})\right)^2=\frac{1}{2}\sum_{i=1}^m \|\Delta_i-d({\bf X}^{(i)})\|_F^2,$$
and the commensurability requires $\binom{m}{2}$ paired distance calculations amongst the $n$ points across the $m$ modalities.
}
\end{remark}

Given a bounded number of Guttman transform updates, the fJOFC algorithm has complexity $O(m^2n^2d)$.  
Contrasting this with the $O((mn)^3)$ complexity of JOFC points to the dramatic speedup achieved by fJOFC; see Section \ref{S:results} for further empirical demonstrations of this computational savings.
We also recall that, even with identical initializations, the JOFC iterates and fJOFC iterates will not agree in general.  The JOFC iterates rely on an approximate computation of {\bf L}$^\dagger$ while the fJOFC iterates utilize an exact algebraically computed {\bf L}$^\dagger$.
Hence, the fJOFC iterates are not only more efficiently computed than the corresponding JOFC iterates, they are also less noisy.

\begin{remark}
\label{rem:parfjofc}
\emph{
  Each step of the fJOFC procedure easily lends itself to parallel computation.
  Implemented in parallel, given a bounded number of Guttman transform updates, fJOFC has complexity $O(m^2n^2d/c)$ when run in parallel over $c$ cores.
}
\end{remark}

\section{Fast out-of-sample embedding for JOFC}
\label{S:foos}
The out-of-sample embedding framework was developed for classical MDS in \cite{trosset2008out} and for Raw Stress MDS in \cite{joos}.
Extending the latter, we
develop the out-of-sample embedding framework for JOFC.
We then demonstrate how this out-of-sample embedding can be dramatically sped-up by exploiting the special structure of the associated JOFC weight matrix, akin to the speedup of fJOFC over JOFC, and empirically demonstrate the efficiency of the procedure in Section \ref{S:ooscaboose}.

Given a configuration $\bX\in\mathbb{R}^{mn\times d}$ obtained via JOFC (or fJOFC) applied to ${\bf\Delta}\in\mathbb{R}^{mn\times mn}$, we observe a new object $\mathcal{O}$, giving rise to the out-of-sample omnibus dissimilarity
\[{\bf \D}^{(o)}\!=\![{\bf\Delta}^{(o)}_{i,j}]\!=\!\begin{bmatrix}
\D_1^{(o)}&\eta&\cdots &\eta \\
\eta&\D_2^{(o)}&\cdots &\eta \\
\vdots & \vdots & \ddots & \vdots\\
\eta&\eta&\cdots &\D_m^{(o)}
\end{bmatrix}\!\!\in\!\mathbb{R}^{m(n+1)\times m(n+1)};\, \D_i^{(o)}\!=\!\begin{bmatrix}
\D_i&\delta_i\\
\delta_i^\top &0
\end{bmatrix}\!\!\!\in\!\mathbb{R}^{n+1\times n+1},\]
where, for each $i\in[m]$, $\delta_i$ represents the within modality dissimilarities between $\mathcal{O}$ and the in sample-data objects for the $i$-th modality.  

While we could run JOFC (or fJOFC) on the full ${\bf \D}^{(o)}$, if $m$ or $n$ is large this often becomes computationally burdensome.
Rather, without re-embedding $\bf \D$, we seek to embed $\mathcal{O}$ into the configuration space determined by $\bX$ so as to best preserve both the matchedness across the $m$ versions of $\mathcal{O}$ and the within modality dissimilarities provided by $\{\delta_i\}_{i=1}^m$.
In the JOFC Raw Stress framework, the out-of-sample raw stress criterion is given by
\begin{equation}
\label{eq:joos}
\sigma_\bX({\bf y})=\underbrace{\sum_{i=1}^m\sum_j(\delta_i(j)-d(\bX^{(i)}_j,{\bf y}_j))^2}_{\text{out-of-sample fidelity}}+w\underbrace{\sum_{i<j} d({\bf y}_i,{\bf y}_j)^2}_{\text{out-of-sample commensurability}},
\end{equation}
where $\by^\top=[\by_1^\top|\by_2^\top|\cdots|\by_m^\top]\in\mathbb{R}^{m\times d}$ is the configuration obtained for the new out-of-sample observation $\mathcal{O}$.

Reordering the rows and columns of ${\bf \D}^{(o)}$ slightly, 
$${\bf\D}^{(o)}=\begin{bmatrix}
\begin{matrix} \mathlarger{\mathlarger{\mathlarger{\mathlarger{\D}}}}\end{matrix}&\begin{matrix}\delta_1&NA&\cdots&NA\\
NA&\delta_2&\cdots&NA\\
\vdots&\vdots&\ddots&\vdots\\
NA&NA&\cdots&\delta_m\end{matrix}\\
\begin{matrix}
\delta_1^\top&NA&\cdots&NA\\
NA&\delta_2^\top&\cdots&NA \\
\vdots&\vdots&\ddots&\vdots\\
NA&NA&\cdots&\delta_m^\top&\end{matrix}&\begin{matrix}0\,\,\,\,\,\,&0\,\,\,&\cdots&0\\
0\,\,\,\,\,\,&0\,\,\,&\cdots&0\\
\vdots\,\,\,\,\,\,&\vdots\,\,\,&\ddots&\vdots\\
0\,\,\,\,\,\,&0\,\,\,&\cdots&0
\end{matrix} 
\end{bmatrix},
$$
we see that the raw stress criterion (\ref{eq:joos}) can be written as 
$$\sigma_\bX({\bf y})=\sum_{i<j} {\bf W}^{(o)}_{i,j}({\bf \Delta}^{(o)}_{i,j}-d_{i,j}({\bf X}^{(o)}))^2,
$$
with the weight matrix ${\bf W}^{(o)}$ and configuration $\bX^{(o)}$ given by (where for $h,k\in\mathbb{Z}>0$, ${\bf 0}_{h,k}$ is the $h\times k$ matrix of all $0$'s 
)
$$\bW^{(o)}=
    \begin{bmatrix}
    {\bf 0}_{mn,mn}&I_m\otimes {\bf1}_n\\
    I_m\otimes {\bf1}_n^\top&wJ_{m,m}-wI_{m}
    \end{bmatrix},
\hspace{5mm}
\bX^{(o)}=\begin{bmatrix}
\bX\\
{\bf y}_1\\
{\bf y}_2\\
\vdots\\
{\bf y}_m
\end{bmatrix}. $$ 
Decompose the Laplacian of $\bW^{(o)}$ via
$$\bl^{(o)}=\bordermatrix{&mn\text{ cols}&m\text{ cols}\cr
                mn\text{ rows}&L_{1,1} &  L_{1,2}\cr 
                m\text{ rows}& L_{1,2}^\top  &  L_{2,2}},$$
and define $B(\bX^{(o)})$ as in Eq. (\ref{eq:B}), a similar decomposition of $B$ is given by 
$$B(\bX^{(o)})=\bordermatrix{&mn\text{ cols}&m\text{ cols}\cr
                mn\text{ rows}&B_{1,1} &  B_{1,2}\cr 
                m\text{ rows}& B_{1,2}^\top  &  B_{2,2}}=\begin{bmatrix}
B_{1,1} & B_{1,2}\\
B_{1,2}^\top & \text{diag}\left( {\bf1}^T \left(\delta_i\circ \frac{1}{d(\bX^{(i)},\by_{(t-1)})}\right)\right)
                \end{bmatrix},$$
                where ``$\circ$'' is the Hadamard product, 
                and for each $j\in[m],$
                $$\frac{1}{d(\bX^{(j)},\by_{(t-1)})}=\left(\frac{1}{d(\bX^{(j)}_1,(\by_{(t-1)})_1)},\ldots,\frac{1}{d(\bX^{(j)}_m,(\by_{(t-1)})_m)}\right)^\top.$$
                Note that
                $$B_{1,2}=-\begin{bmatrix}
\delta_1\circ\frac{1}{d(\bX^{(1)},\by_{(t-1)})}&{\bf 0}_m&\cdots&{\bf 0}_m\\
{\bf 0}_m&\delta_2\circ \frac{1}{d(\bX^{(2)},\by_{(t-1)})}&\cdots&{\bf 0}_m\\
\vdots&\vdots&\ddots&\vdots\\
{\bf 0}_m&{\bf 0}_m&\cdots&\delta_m \circ \frac{1}{d(\bX^{(m)},\by_{(t-1)})}\\
                \end{bmatrix}.$$
A similar majorization argument to that of in-sample JOFC yields the out-of-sample embedding procedure:
\begin{itemize}
\item[1.] Initialize the out-of-sample configuration at a random initialization $\by={\bf y}_{(0)}$.
\item[2.] While $\sigma_\bX({\bf y}_{(t)})-\sigma_\bX({\bf y}_{(t-1)})>\epsilon$ for a predetermined threshold $\epsilon$, update $\by_{(t)}$ via the Guttman transform:
\begin{equation}
\label{eq:joosupdate}
\by_{(t)}=L_{2,2}^\dagger(B_{1,2}^\top-L_{1,2}^\top)\bX+L_{2,2}^\dagger \cdot\text{diag}\left( {\bf1}^T \left(\delta_i\circ \frac{1}{d(\bX^{(i)},\by_{(t-1)})}\right)\right) \by_{(t-1)}.
\end{equation}
Derivation of this update via majorization is completely analogous to the derivation of the JOFC update step, and so details are suppressed.
\end{itemize}

As $L_{2,2}=(n+mw)I_m-w J_m$, it is immediate that 
$L_{2,2}^{\dagger}=\frac{1}{n+mw}I_m+\frac{w}{n(n+mw)} J_m.$
Therefore, to efficiently compute (\ref{eq:joosupdate}), we: 
\begin{itemize}
\item[1.] For each $j\in[m]$, compute 
$$\xi_j:=\left(-\delta_j\circ\frac{1}{d(\bX^{(j)},\by_{(t-1)})}+{\bf 1}_n\right)^\top \bX^{(j)},$$
and
$$\psi_j:= {\bf1}^T \left(\delta_i\circ \frac{1}{d(\bX^{(i)},\by_{(t-1)})}\right).$$
For each $j\in[m]$, this vector-matrix multiplication has complexity $O(nd)$,
and the full complexity of this step is $O(nmd)$.  
\item[2.] Routine computations then yield the following simplification of the Guttman transform update:
\begin{align*}
(\by_{(t)})_j=&\frac{\xi_j}{n+mw}+\frac{w}{n(n+mw)}\sum_{k=1}^m \xi_k\\
&+\frac{\psi_j}{n+mw} (\by_{(t-1)})_j+\frac{w}{n(n+mw)}\sum_{k=1}^m \psi_k\cdot (\by_{(t-1)})_k
\end{align*}
each of which has complexity $O(d),$ and the full complexity of this step is $O(md)$
\end{itemize}
Given a fixed number of modalities $m$ and a bounded number of iterates in the algorithm, the complexity of embedding each new out-of-sample observation is linear in $n$, allowing for this out-of-sample procedure to be efficiently implemented on very large data sets.
We note that the details for simultaneously embedding $k>1$ out-of-sample points are completely analogous to the $k=1$ case and so are omitted.
\begin{figure}[t!]
\centering
\includegraphics[width=1\textwidth]{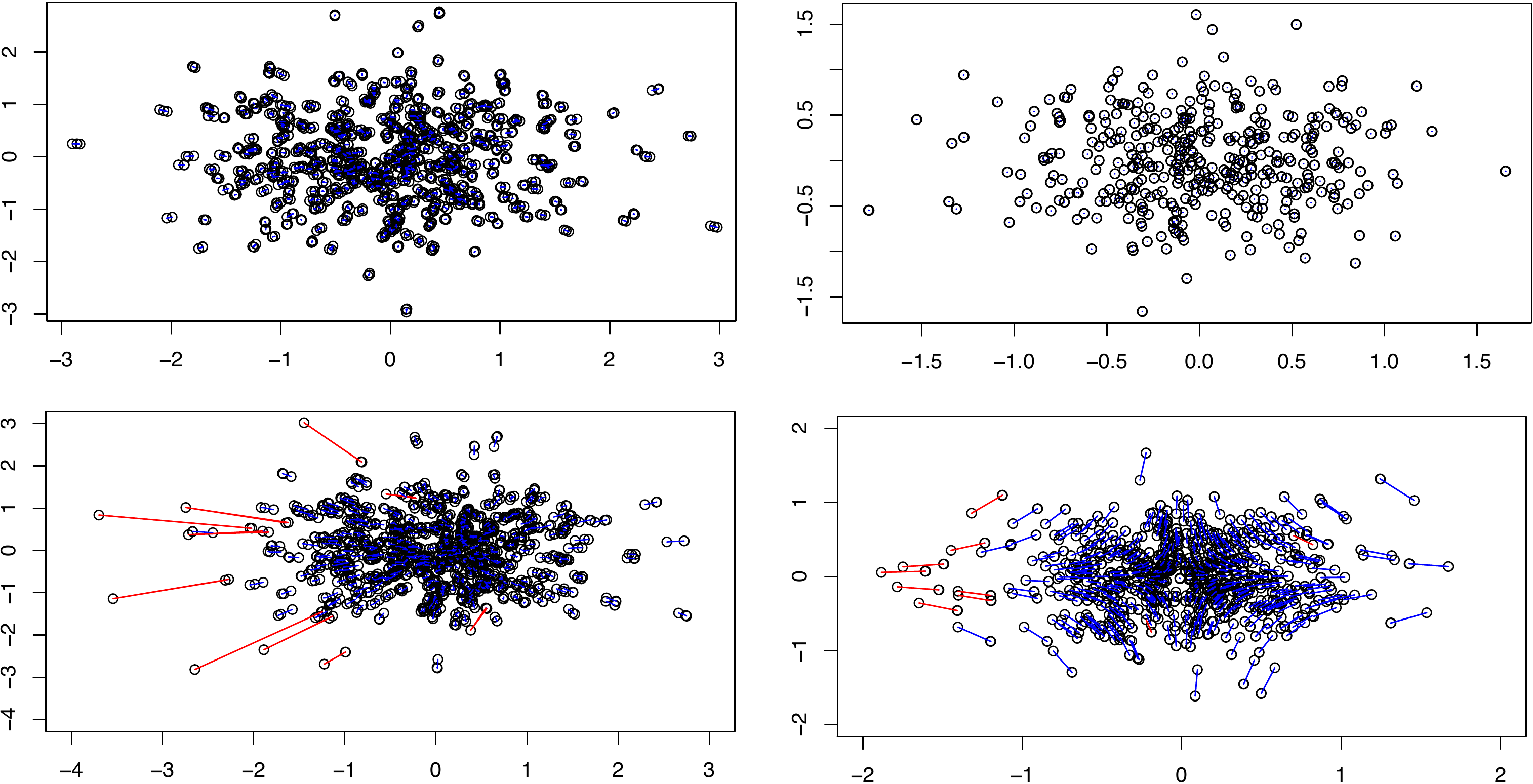}
\caption{For a single Monte Carlo iterate, we plot the embeddings of the three dissimilarities in the matched (top row, left panel is fJOFC, right panel is 3-MDS) and the anomaly (bottom row, left panel is fJOFC, right panel is 3-MDS) settings.  In the matched setting, matched triplets are connected by blue lines.  In the anomaly setting, blue lines connect the matched points across the embeddings and red lines connect the ten anomaly points they are ``matched'' to in the data.}
\label{fig:3MDSJOFC}
\end{figure}
\section{Results}
\label{S:results}
In this section we both compare and contrast fJOFC and 3-RSMDS and demonstrate the dramatic run time increase achievable by fJOFC versus JOFC over a variety of real and simulated data examples; note that all run times are measured in seconds.  In all examples, the algorithms were implemented on a MacBook Pro with a 2.6 GHz Intel Core i5 processor and 4GB 1600 MHz DDR3 memory.

\subsection{3-RSMDS and fJOFC}
\label{S:3J}
As mentioned previously, fJOFC can be viewed as a softly constrained version of 3-RSMDS.  Herein, through a simple illustrative experiment, we highlight the advantages (and disadvantages) of the fJOFC framework.
\begin{table}[t!]
\centering
\begin{tabular}{| l ||c |c |c|c|c|c|c|c|}
\hline
  Method & Runtime1 & Stress1 &  ARI & Runtime2 &Stress2 &ARI2&Conf. Ratio\\
  \hline\hline
  3-RSMDS & 214.52 &0.042 & 0.69& 312.39  &  0.18& 0.40&10.29\\
  \hline
  fJOFC & 1.39 & 0.03 & 0.66&11.56  &  0.16 & 0.57&76.07\\
  \hline
\end{tabular}
\caption{The average running time (over 25 MC iterates) is shown as Runtime1 (in the matched setting) and Runtime2 (in the anomaly setting). The average final normalized stress is Stress1 (in the matched setting) and Stress2 (in the anomaly setting).  In the matched setting, the ARI gives the a measure of the fidelity of the $K$-means clustering of the data into 400 clusters (each should contain the three jitters of the same point).  In the anomaly setting, the ARI2 column gives a measure of the fidelity of the $K$-means clustering of the non-anomalous data into 390 clusters (each should contain the three jitters of the same point).
Lastly, the Conf. Ratio column gives the ratio of the average distance between the triplets of points that have the anomalies (the ten outlier triplets) and the triplets that are correctly matched in the anomaly setting (the 390 non-outlier triplets).}
\label{table:table1}
\end{table}
Let $Y\in\mathbb{R}^{400\times2}$ have rows which are independent 2-dimensional Gaussian$\big((5,5),I_2\big)$ random variables. 
Letting $z=\max(Y)-\min(Y)$, for $i=1,2,3,$ we set $Y_i$ to be $Y+E_i$, with the entries of $E_i$ being independent Uniform$(-z/50,z/50)$ random variables, which are also independent across $i$. We set $\Delta_i$ to be the interpoint distance matrix of $Y_i$. 
These $\{Y_i\}$ represent our $n=400$ objects measured under $m=3$ modalities. 
Let $Z\in\mathbb{R}^{400\times2}$ have rows $11,12,\ldots,400$ identical to those in $Y$ and let the first ten rows of $Z$ be independent 2-dimensional Gaussian$\big((8,8),2\cdot I_2\big)$ random variables.
Let $Y_4=Z+E_4$, with $E_4$ defined analogously to the $E_i's$ above. Let $\Delta_4$ be the interpoint distance matrix of $Y_4$. 

We use fJOFC and the INDSCAL algorithm (for 3-RSMDS, as implemented in the \texttt{smacof} package \citep{leeuw2008multidimensional} in \texttt{R}) to embed $(\Delta_1,\Delta_2,\Delta_3)$ (the matched setting) and $(\Delta_1,\Delta_2,\Delta_4)$ (the anomaly setting).  Results are summarized below in Figure \ref{fig:3MDSJOFC} and Table \ref{table:table1}.  In Figure \ref{fig:3MDSJOFC}, for a single Monte Carlo iterate, we plot the embeddings of the three dissimilarities in the matched (top row, left panel is fJOFC, right panel is 3-MDS) and the anomaly (bottom row, left panel is fJOFC, right panel is 3-MDS) settings.  In the matched setting, matched triplets are connected by blue lines.  In the anomaly setting, blue lines connect the matched points across the embeddings and red lines connect the ten anomaly points they are ``matched'' to in the data.

Results over 25 MC iterates are summarized in Table \ref{table:table1}.
The average running time is shown as Runtime1 (in the matched setting) and Runtime2 (in the anomaly setting). The average final normalized stress is Stress1 (in the matched setting) and Stress2 (in the anomaly setting).  In the matched setting, the ARI (adjusted Rand index; see \cite{hubert1985comparing}) a measure of the fidelity of the $K$-means clustering of the data into 400 clusters (each should contain the three jitters of the same point).  
In the anomaly setting, the column ARI2 gives the a measure of the fidelity of the $K$-means clustering of the non-anomalous data into 390 clusters (each should contain the three jitters of the same non-anomalous point).
An ARI of 1 means that the clustering of the embedded points perfectly clusters the repeated observations of the data, while an ARI of 0 indicates that the clustering of the embedded points behaves as chance in recovering the clusters of the repeated observations.
 Lastly, the Conf. Ratio column gives the ratio of the average distance between the triplets of points that have the anomalies (the ten anomaly triplets) and the triplets that are correctly matched in the anomaly setting (the 390 non-anomaly triplets).
\begin{figure}[t!]
     \begin{center}
        \subfigure[Matched Setting.]{%
            \label{fig:ME}
            \includegraphics[width=0.5\textwidth]{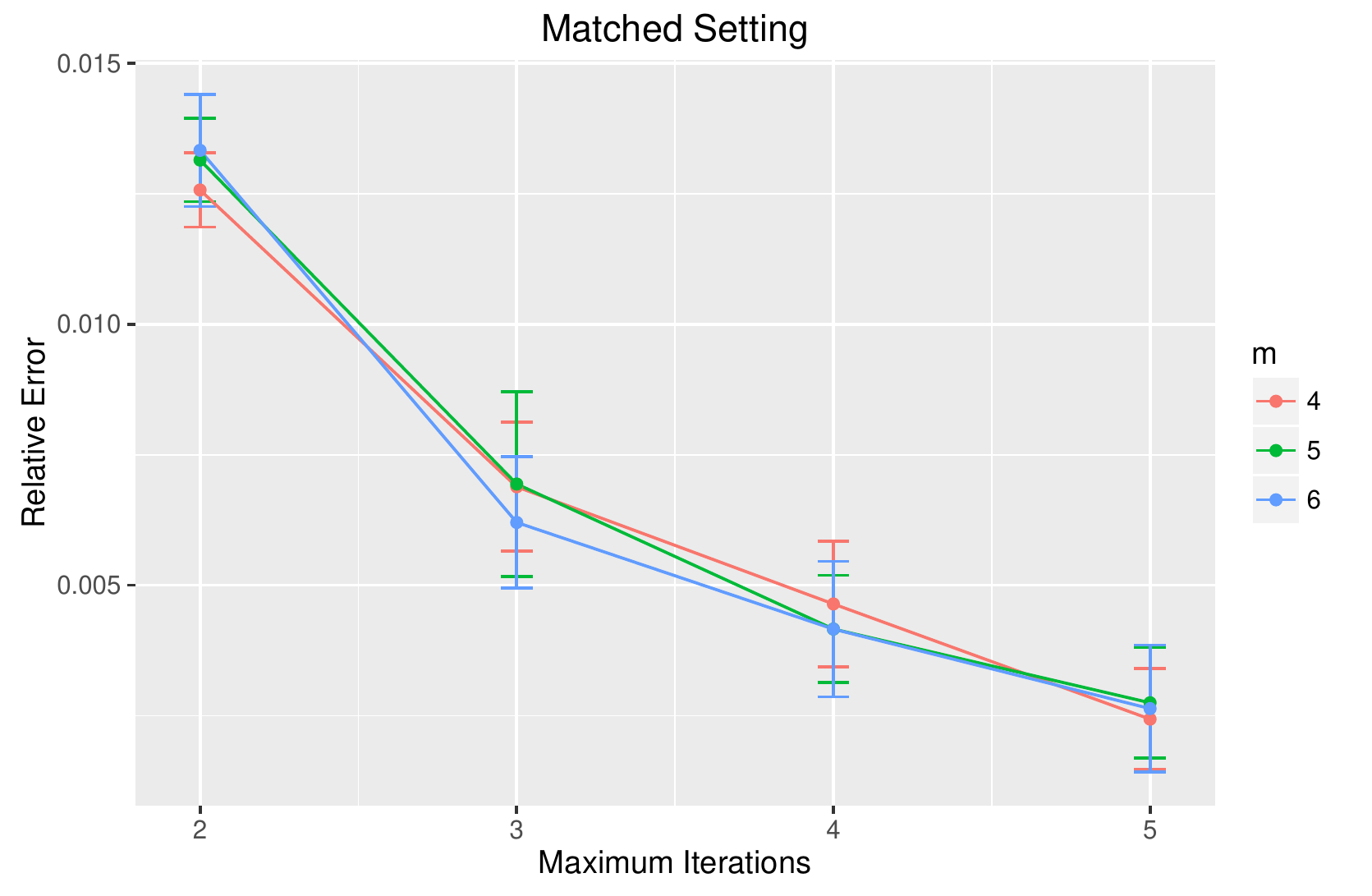}
        }%
        \subfigure[Anomaly Setting.]{%
           \label{fig:MEA}
           \includegraphics[width=0.5\textwidth]{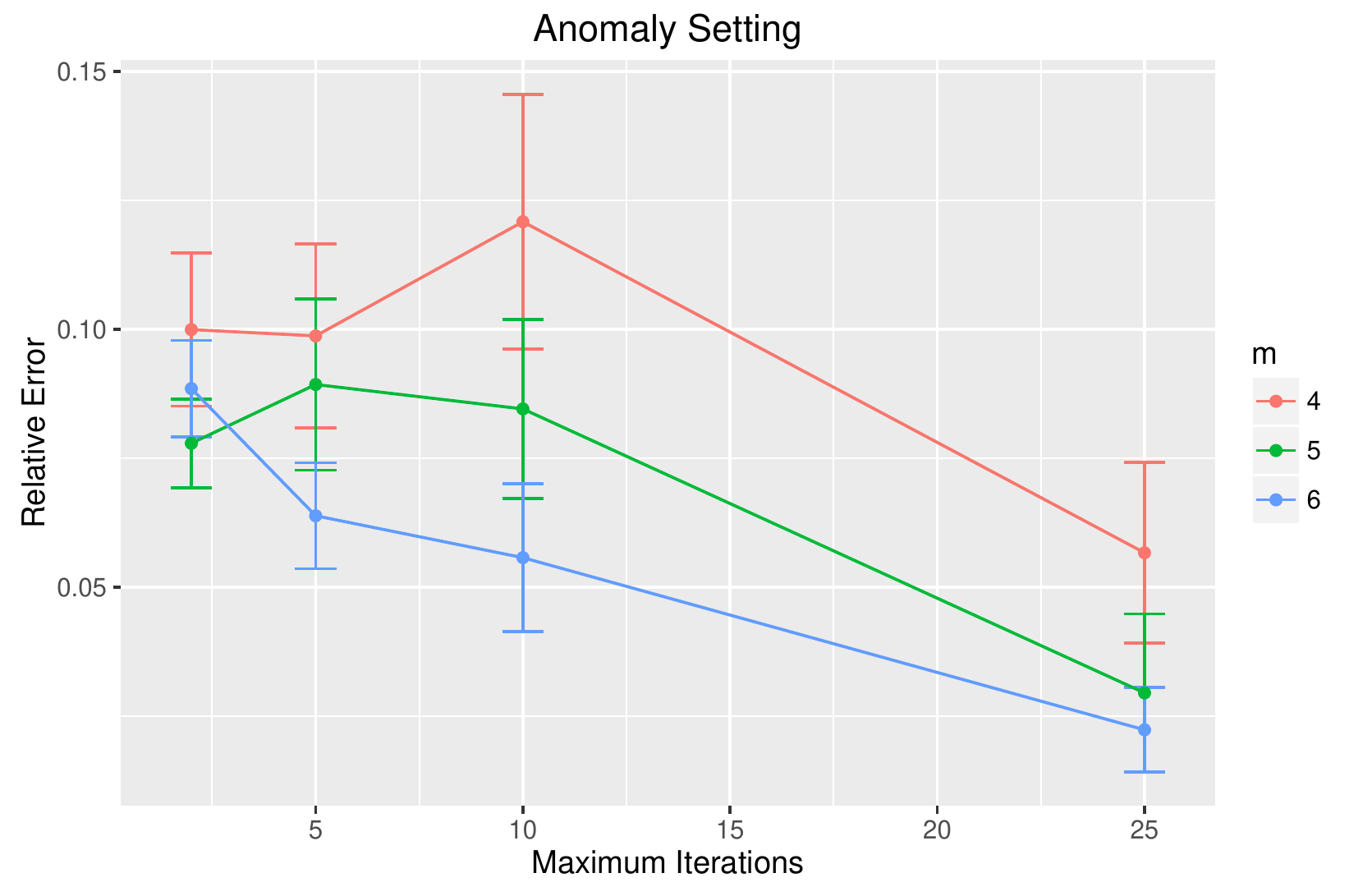}
        }
    \end{center}
    \caption{%
        We plot the average (over 25 Monte Carlo iterates) relative error 
$\frac{\|{\bf X}_{(\text{final})}-{\bf X}_{(k)}\|_2}{\|{\bf X}_{(\text{final})}\|_2}$ $\pm 2$ s.e.
over a range of values of $k$ (the $x$-axis in Figure \ref{fig:ME}).  In the left panel, we plot the ratio in the matched setting with $m=4,5,6$ dissimilarities and in the right panel, we plot the ratio in the anomaly setting with $m=4,5,6$ dissimilarities, one of which contains the anomaly.
Note that in the anomaly setting, only the relative error amongst the $n*(m-1)+(n-10)$ non-anomalous points is plotted.
     }%
   \label{fig:mmee}
\end{figure}
From this simple experiment, we see that
fJOFC is empirically i. much faster than (this off the shelf implementation of) 3-RSMDS; ii. performs comparably to 3-RSMDS when the data are all matched across the modalities with no anomalous behavior---see the ARI column in Table \ref{table:table1}; iii. is better able to preserve the correct matchedness in the presence of anomalous data---see the ARI2 and Conf. Ratio columns of Table \ref{table:table1}; results which are echoed in \cite{sunpriebe2013,shen2014manifold}

\subsection{ Error tolerance}
\label{s:err}
With the same setting as in Section \ref{S:3J}, we explore the effect of early stopping on the global fJOFC output.  As mentioned previously, the sequential Guttman transforms often exhibit good global properties, and good solutions can often be obtained after only a few iterates.  To demonstrate this, we plot the relative error (over 25 Monte Carlo iterates)
$\frac{\|{\bf X}_{(\text{final})}-{\bf X}_{(k)}\|_2}{\|{\bf X}_{(\text{final})}\|_2}$ $\pm2$ s.e.
over a range of values of $k$ (the $x$-axis in Figure \ref{fig:ME}).  In the left panel, we plot the ratio in the matched setting with $m=4,5,6$ dissimilarities and in the right panel, we plot the ratio in the anomaly setting with $m=4,5,6$ dissimilarities, one of which contains the anomaly.
Note that in the anomaly setting, only the relative error amongst the $n*(m-1)+(n-10)$ non-anomalous points is plotted.
We see that, in the matched setting, very few sequential iterates are needed before the embedding stabilizes.  In the anomaly setting---when on average over 100 sequential iterates are needed for the algorithm to terminate with $\epsilon=10^{-6}$ tolerance---we see $\approx 5\%$ relative error with only 25 iterates.  
Indeed, here and in the real data examples, we find that $\epsilon=10^{-6}$ is often a conservative tolerance level and a sufficiently good embedding can be obtained with far fewer iterates; we are presently investigating methods for adaptively choosing the number of iterates, though we do not pursue this further here.

\subsection{JOFC versus fJOFC}

\begin{figure}[t!]
\hspace{10mm}
\centering
\includegraphics[width=1\textwidth]{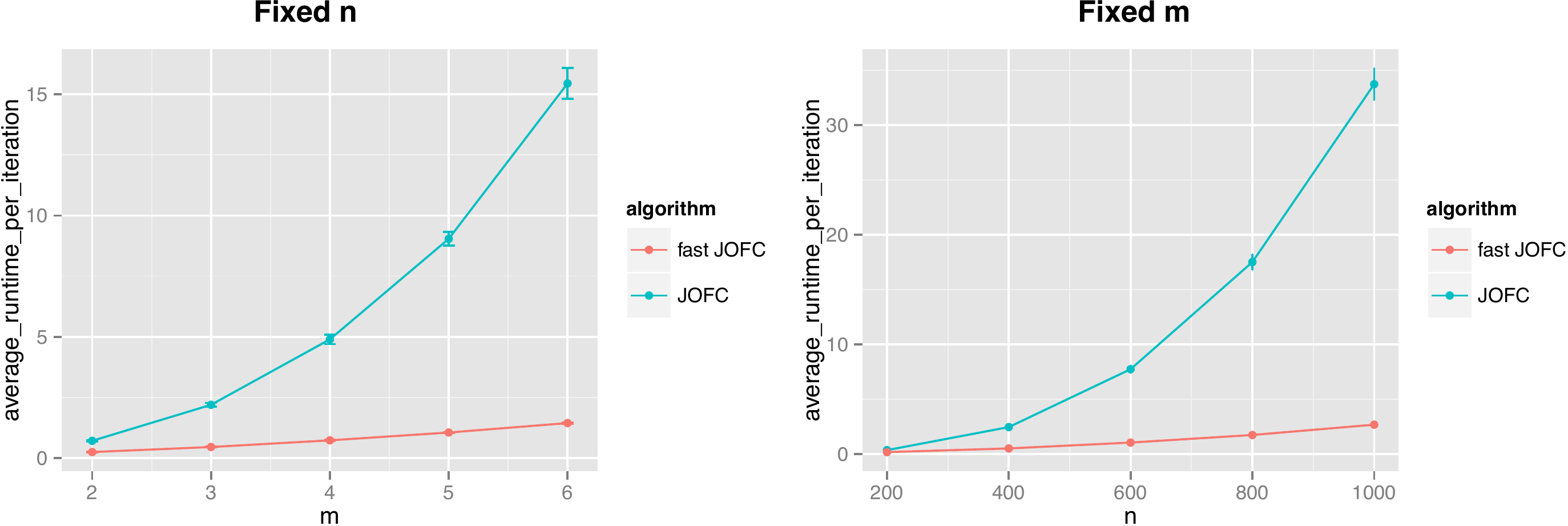}
\caption{We embed ${\bf \Delta}\in\mathbb{R}^{nm\times nm}$ via fJOFC and JOFC using identical initial configurations ${\bf X}_{(0)}=$cMDS$({\bf \Delta})$ as in Remark \ref{rem:oldJOFC}.  We then plot the average run time (in seconds) per iteration ($\pm2s.e.$) versus $m$ (left panel) and $n$ (right panel) for both JOFC and fJOFC, averaged over 50 Monte Carlo replicates.  In the left panel we fix $n=400,$ and vary $m=2,3,4,5,6$.  In the right panel, we fix $m=3$ and  vary $n=200,400,600,800,1000$.}
\label{fig:fast12}
\end{figure}


Let $Y\in\mathbb{R}^{400\times2}$ have rows which are independent 2-dimensional Gaussian$\big((5,5),I_2\big)$ random variables. 
Letting $z=\max(Y)-\min(Y)$, for $i=1,2,\ldots,6,$ we set $Y_i$ to be $Y+E_i$, with the entries of $E_i$ being independent Uniform$(-z/50,z/50)$ random variables, which are also independent across $i$. We set $\Delta_i$ to be the interpoint distance matrix of $Y_i$. 
These $\{Y_i\}$ represent our $n=400$ objects measured under $m=6$ modalities. 
For $m=2,3,\ldots,6,$ we embed the omnibus matrix $\bf \Delta$ (defined as in Section \ref{S:JOFC})
\noindent into $\mathbb{R}^2$ with both fJOFC (in serial) and JOFC using an identical initial configurations ${\bf X}_{(0)}=\text{cMDS}({\bf \Delta}),$ as outlined in Remark \ref{rem:oldJOFC}.  We plot the average run time per iteration versus $m$ for both fJOFC and JOFC in Figure \ref{fig:fast12} (left panel), averaged over 50 Monte Carlo replicates.
Even in this relatively small simulation, the decreased runtime speed is dramatically illustrated, even with fJOFC run in serial.  
The ratio of the average run times (JOFC versus fJOFC) is
$(2.86,  4.82 , 6.70 , 8.59, 10.71)$ for $m=(2,3,4,5,6)$, which suggests that fJOFC is a factor of $m$ ($\approx1.6m$) faster than JOFC here.
This corroborates the runtime results in Section \ref{sec:fastJOFC}; indeed, as here $n$ is constant, JOFC has complexity $O(m^3)$ while fJOFC has complexity $O(m^2)$.

We next consider the case of fixed $m=3$ and varying $n=(200,400,600,800,1000)$.  With $Y$ and $\bf \Delta$ defined as above, we again embed ${\bf \Delta}\in\mathbb{R}^{nm\times nm}$ into $\mathbb{R}^2$ via fJOFC (in serial) and JOFC using identical initial configurations ${\bf X}_{(0)}=$cMDS$({\bf \Delta})$.  
In Figure \ref{fig:fast12} (right panel), we plot the average run time per iteration versus $n$ for both JOFC and fJOFC, averaged over 50 Monte Carlo replicates.
Again, note the dramatic speedup achieved by fJOFC, with the ratio of the average run times (JOFC versus fJOFC) being
$(2.10,  4.86,  7.45, 10.13, 12.63)$ for $n=(200,400,600,800,1000)$. 
This suggests that fJOFC is a factor of $n$ ($\approx0.12n$) faster than JOFC here, which corroborates the runtime results in Section \ref{sec:fastJOFC}; indeed, as here $m$ is constant, JOFC has complexity $O(n^3)$ while fJOFC has complexity $O(n^2)$.

\subsection{Out-of-sample efficiency}
\label{S:ooscaboose}

\begin{figure}[t!]
\centering
\includegraphics[width=1\textwidth]{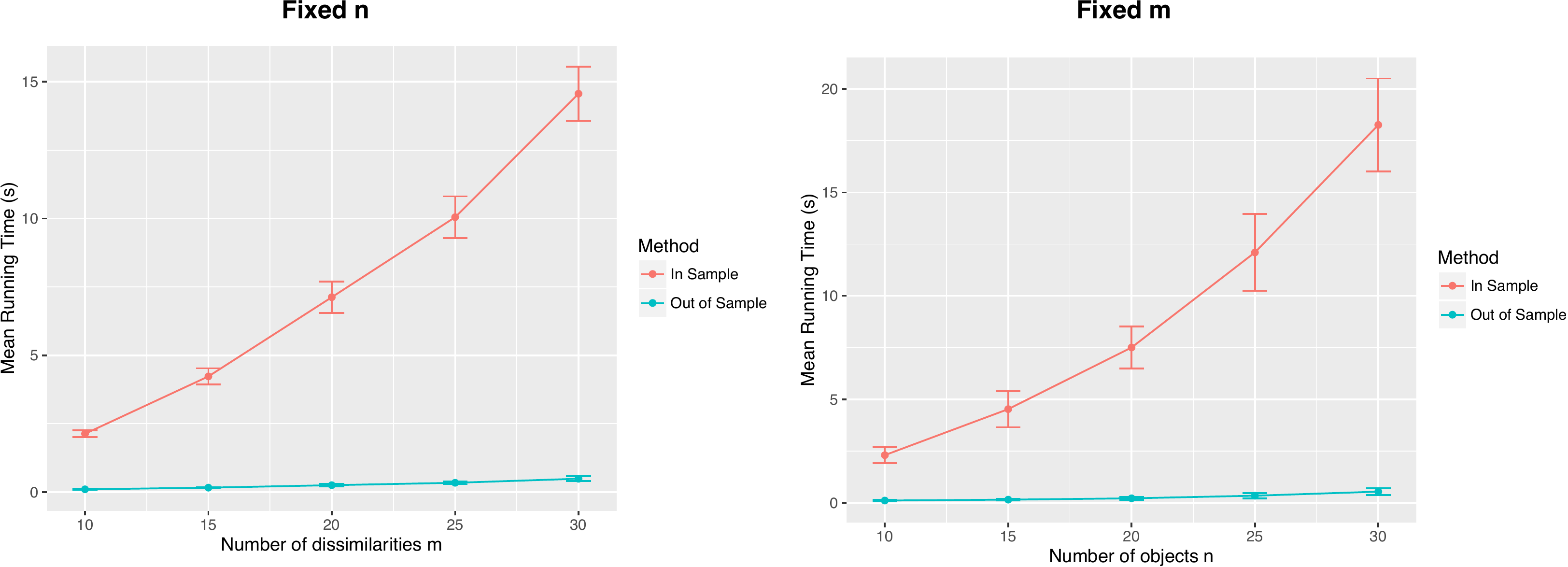}
\caption{We embed all but one object of ${\bf \Delta}\in\mathbb{R}^{nm\times nm}$ via fJOFC, and use the out-of-sample procedure to embed the final object ($m$ views of the $n$-th object). We then plot the average running time (in seconds) $\pm2s.e.$ of the in-sample and the out-of-sample procedure versus $m$ (left panel) and versus $n$ (right panel), averaged over 25 Monte Carlo replicates.  In the left panel we fix $n=200,$ and vary $m=10,15,20,25,30$.  In the right panel, we fix $m=10$ and  vary $n=200,300,400,500,600$.}
\label{fig:oosINS}
\end{figure}
We next demonstrate the efficiency of the out-of-sample fJOFC procedure.
With the same data set-up as above (with $3$-dimensional Gaussian random variables here, but otherwise identical to the data setup used above), we embed all but one object of ${\bf \Delta}\in\mathbb{R}^{nm\times nm}$ via fJOFC, and use the out-of-sample procedure to embed the final object ($m$ views of the $n$-th object). 
Running time results (in seconds) are plotted in Figure \ref{fig:oosINS}, where we plot the average running time (in seconds) $\pm2s.e.$ of the in-sample and the out-of-sample procedure versus $m$ (left panel) and versus $n$ (right panel), averaged over 25 Monte Carlo replicates.  In the left panel we fix $n=200,$ and vary $m=10,15,20,25,30$.  In the right panel, we fix $m=10$ and  vary $n=200,300,400,500,600$.
As seen previously, the runtime of fJOFC empirically varies quadratically (in $n$ for fixed $m$ and in $m$ for fixed $n$).
However, we observe that the runtime of the out-of-sample procedure empirically varies linearly (in $n$ for fixed $m$ and in $m$ for fixed $n$), which agrees with the computational complexity results of Section \ref{S:foos}.

In Table \ref{table:table2} we show the sum of the residual errors of the out-of-sample embedding versus the in-sample embedding---$\sum_{i=1}^{m}\|{\bf X}^{(i)}_{(\text{final})}[n,:]-{\bf y}_i\|_2$---for fixed $n$ and varying $m$ (top row) and for fixed $m$ and varying $n$ (bottom row) averaged over 25 Monte Carlo iterates. 
For each combination of $m$ and $n$, we first embed the full $mn\times mn$ dissimilarity ${\bf \Delta}$ using fJOFC.
We next embed all but one of the objects ($n-1$ objects over $m$ modalities) using fJOFC and the $n$-th object via the out-of-sample procedure of Section \ref{S:foos}, and compute the sum of the residual errors between the out-of-sample and the in-sample embeddings of the $n$-th object.
We see that, for fixed $n$ and varying $m$, the total error is increasing in $m$ but negligible on average per modality.  As $m$ is fixed and $n$ varies, the total error is relatively constant, which is unsurprising as, in each case, exactly $m$ additional data points are being out-of-sample embedded into a fixed dimensional space.

\begin{table}[t!]
\caption*{Sum of Residual Errors of Out-of-Sample Versus In-Sample}  
\centering
\begin{tabular}{| l ||c |c |c|c|c|c|c|}
\hline
   & $m=10$ & $m=15$ & $m=20$ & $m=25$ &$m=30$\\
  \hline
  $n=200$ & 0.067 & 0.121 & 0.184 & 0.366 & 0.364 \\
  \hline
  \hline
   & $n=200$ & $n=300$ & $n=400$ & $n=500$ &$n=600$\\
  \hline
  $m=10$ & 0.057 & 0.059 & 0.101 & 0.078 & 0.091 \\
  \hline
\end{tabular}
\caption{The sum of the residual errors of the out-of-sample embedding versus the in-sample embedding---$\sum_{i=1}^{m}\|{\bf X}^{(i)}_{(\text{final})}[n,:]-{\bf y}_i\|_2$---for fixed $n$ and varying $m$ (top row) and for fixed $m$ and varying $n$ (bottom row) averaged over 25 Monte Carlo iterates. }
\label{table:table2}
\end{table}

\subsection{Real Data Examples}
We next demonstrate the key feature of the JOFC procedure in a pair of real data sets; namely, the ability of the algorithm to preserve cross-modality matchedness while not forcing incommensurate versions of the data points to be artificially embedded close to one another. 
Indeed, in the JOFC procedure,
\begin{itemize}
\item[1.] if an object's properties are well-preserved across the $m$ modalities, then the object's associated $m$ points in the configuration will be embedded close to each other; 
\item[2.] if an object's properties are not well-preserved across the $m$ modalities, then JOFC (with well-chosen $w$) will not artificially force the object's $m$ incommensurate configuration points to be close to each other in the embedding. 
\end{itemize}
Incommensurate embeddings can inform both how and why the data modalities differ.  
By studying these pathologies further, we aim to better understand the data features that are emphasized in one modality versus another, which is crucial for understanding potential benefits from pursuing further inference in the joint (versus single) embedding space.

We explore this further below in a data set derived from the French and English Wikipedia graphs and in a time series of zebrafish calcium ion brain images from \cite{prevedel2014simultaneous}.
\begin{figure}[t!]
     \begin{center}
        \subfigure[Dendrogram merge heights, all $n=1382$ points.]{%
            \label{fig:DMH}
            \includegraphics[width=0.5\textwidth]{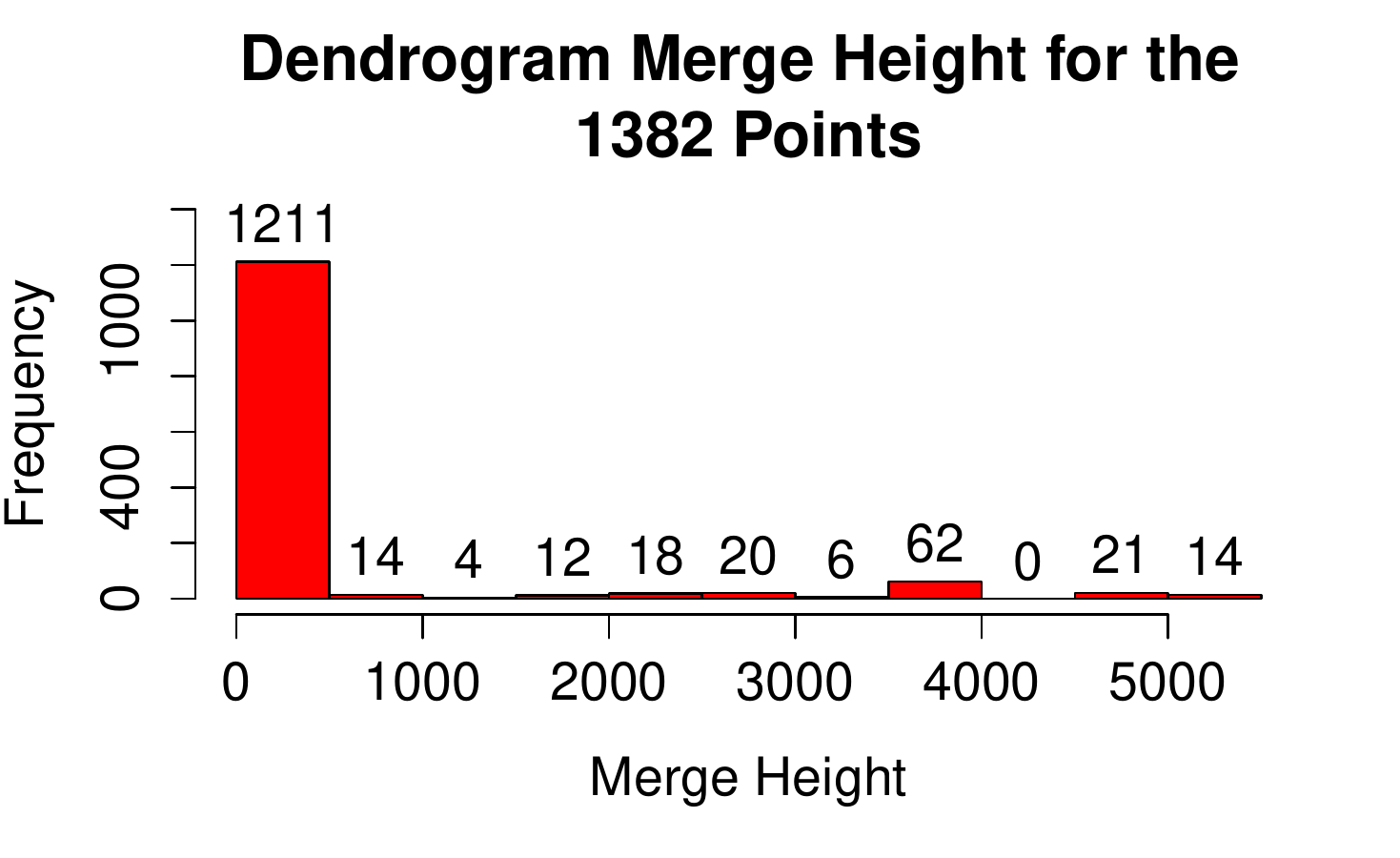}
        }%
        \subfigure[Dendrogram merge heights for the $n=1055$ points with merge height $<100$.]{%
           \label{fig:DMZ}
           \includegraphics[width=0.5\textwidth]{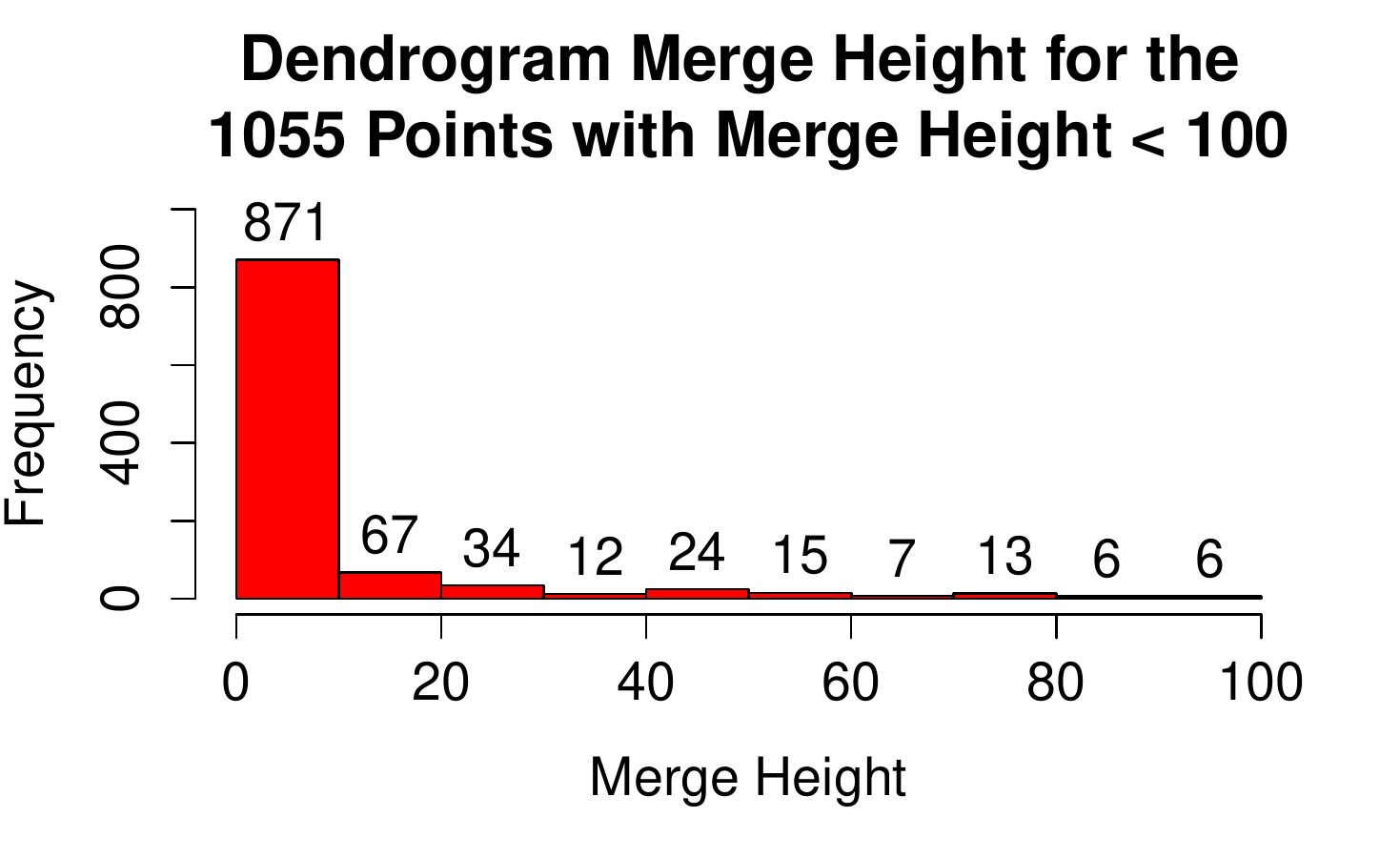}
        }
    \end{center}
    \caption{%
        Histograms showing, for each of the $n=1382$ points in (a) and for each of the $1055$ points with merge height $<100$ in (b),
         the height in the hierarchical clustering dendrogram when each of the four modalities was first merged into a single cluster for that point.
     }%
   \label{fig:subfigures}
\end{figure}
\subsubsection{Wikipedia}
\label{S:wiki}
We collect the $n=$1382 articles $\{y_{1i}\}_{i=1}^{1382}$ from English Wikipedia which compose the 2-hop neighborhood of the article entitled ``Algebraic Geometry'' (where articles are linked if there exists a hyperlink in one article to the other, and these links are considered undirected).
There is a natural $1$-$1$ correspondence between these articles and their versions in French Wikipedia, and we will denote the associated French articles by $\{y_{2i}\}_{i=1}^{1382}$.  

As in \cite{shen2014manifold}, each $\{y_{ji}\}_{i=1}^{1382}$ for $j=1,2$, further gives rise to two measures of inter-article dissimilarity:  $\D_{j1}$, the shortest path distance in the undirected hyperlink graph; and $\D_{j2}$, the cosine dissimilarities between text feature vectors (provided by latent semantic indexing, see \cite{deerwester1990indexing} for detail) associated with each article.
We use fJOFC---with $w=10$ as suggested by \cite{adali2013fidelity}---to embed these $n=1382$ points across $m=4$ modalities into $\mathbb{R}^{10}.$ 
Note that implementing our fJOFC algorithm in serial ran in $\approx$42.2 minutes while the JOFC algorithm with the same settings ran in $\approx$10.37 hours (a factor of $\approx$14.7 speedup).

In this omnibus embedding, if all 4 embedded versions of a single Wikipedia article lie close together, then this article's relationship to all of the other articles is preserved across modality.
If any of the 4 embedded versions is incommensurate with the others then this would indicate either: 
\begin{itemize}
\item[i.] The text features of the article differ significantly across language; i.e. the associated row of $\bX_i^{(2)}$, the embedding associated with $\D_{12}$, is far from $\bX_i^{(4)}$, the embedding associated with $\D_{22}$.
While the French articles are not translations of their English counterparts (or vice versa), further understanding the textual feature highlighted by these incommensurabilities would be useful before pursuing further inference (e.g. topic modeling) in the joint embedding.
\item[ii.] The hyperlink graph structure is not preserved across modality; i.e. the associated row of $\bX_i^{(1)}$, the embedding associated with $\D_{11}$, is far from $\bX_i^{(3)}$, the embedding associated with $\D_{21}$.

\item[iii.] The hyperlink structure and the textual similarities are incommensurate; i.e. the associated row of $\bX_i^{(1)}$, the embedding associated with $\D_{11}$, is far from $\bX_i^{(2)}$, the embedding associated with $\D_{12}$, or the associated row of $\bX_i^{(3)}$, the embedding associated with $\D_{21}$, is far from $\bX_i^{(4)}$, the embedding associated with $\D_{22}$.
By studying these incommensurabilities further, we hope to better understand the data features that are emphasized by graph-based versus text-feature-based methodologies.
\end{itemize}

To investigate further, we proceed by hierarchically clustering (using Ward's method, see \cite{johnson1967hierarchical} for detail) the $4\times 1382$ points of the omnibus embedding and then compute the pairwise cophenetic distance (the height in the resulting dendrogram at which the two points are first clustered together) between each of the points.  
If the dissimilarities are well preserved across modality, then the maximum cophenetic distance between two embedded versions of the same article (we call this the {\it Dendrogram Merge Height} or DMH) should be small.  
\begin{figure}[t!]
\hspace{10mm}
\centering
\includegraphics[width=.6\textwidth]{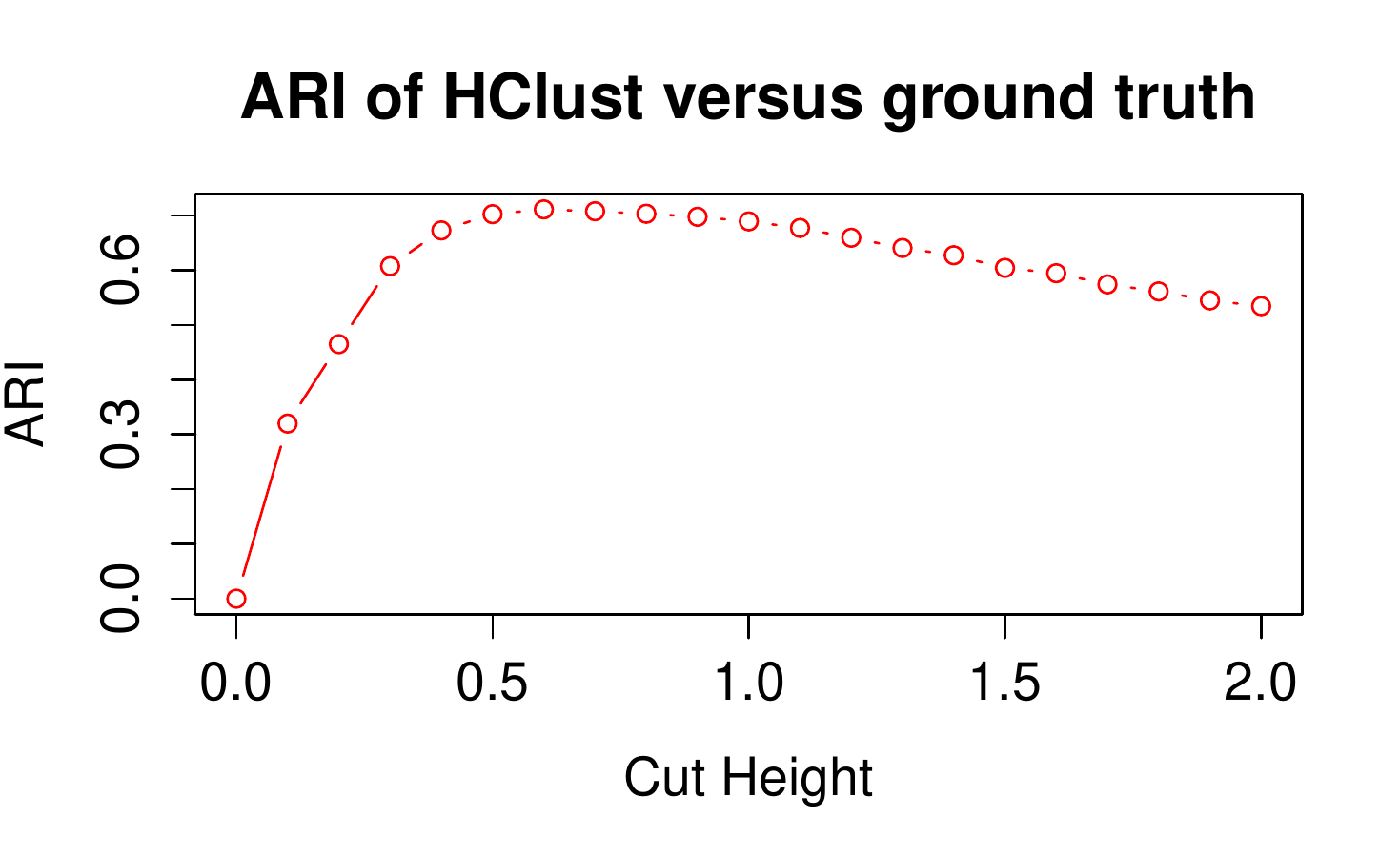}
\caption{The adjusted Rand index between the clusters given by the hierarchical clustering dendrogram at height $h\in[0,2]$ and the ground truth clustering (given by the 1382 size 4 clusters each composed of a single article across modalities).}
\label{fig:ARI}
\end{figure}

In Figure \ref{fig:DMH}, we plot a histogram of the DMH's for the 1382 articles, and note that over 76$\%$ of the articles have DMH less than 100.  In Figure \ref{fig:DMZ} we see that over 63$\%$ of the articles have DMH less than 10.
To further confirm that the dissimilarities are well preserved across modality, we calculated cluster labels given by the hierarchical clustering dendrogram at height $h\in[0,2]$.  We then compute the adjusted Rand index, ARI \citep[see][]{hubert1985comparing}, between these clusterings and the ground truth clustering (given by the 1382 size 4 clusters each composed of a single article across modalities), and plot this in Figure \ref{fig:ARI}.  From the figure, we see that the clustering is not only grossly clustering the article 4-tuples together, but is also capturing the fine-grain differences between the different articles as well.

If the ARI between the hierarchical clustering and the ground truth clustering was equal to 1, then the structure of the four dissimilarities would be nearly identical, and joint inference across modality would yield minimal gain over separately embedding the $\D_i$'s and then applying subsequent inference methodologies.  \begin{figure}[t!]
\hspace{10mm}
\centering
\includegraphics[width=.7\textwidth]{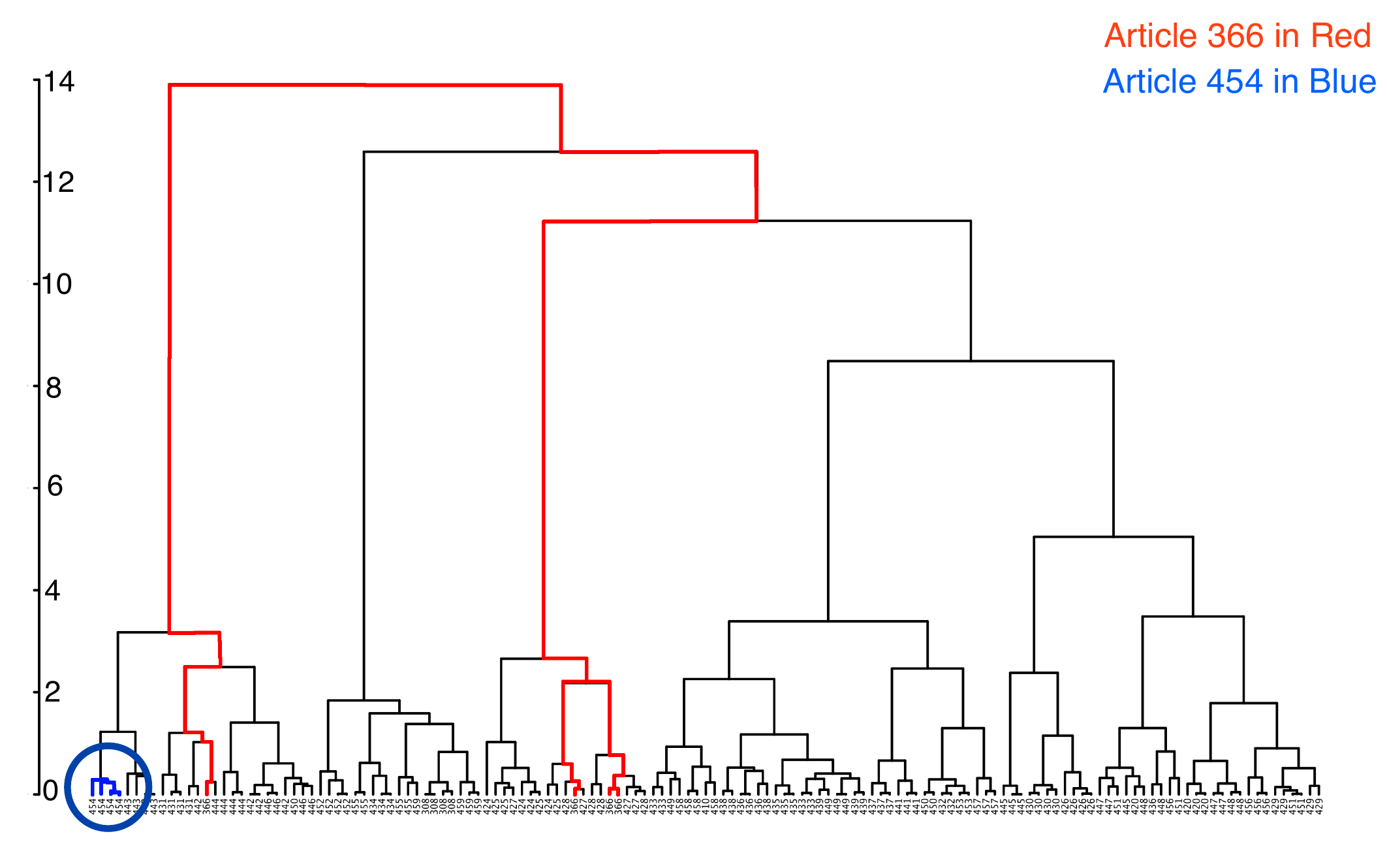}
\caption{A branch of the hierarchical clustering dendrogram when the tree is cut at height 20.  Note that the four modality-specific embeddings of article 454 (highlighted in blue in the dendrogram) are very similar, while those of article 366 are not (the English graph with shortest path distance differs significantly from the other three modalities for this point).}
\label{fig:dend}
\end{figure}
From Figures \ref{fig:DMH}-\ref{fig:DMZ} and \ref{fig:ARI}, we see this is not the case. 
Indeed, we see that the text-feature-based methods and graph-based methods are emphasizing some different data features both within and across language, and therefore
for some articles the relative geometry in the four modality-specific embeddings is not commensurate.
We illustrate this in Figure \ref{fig:dend}, where we plot a
branch of the hierarchical clustering dendrogram when the tree is cut at height 20.  
Note that although the four modality-specific embeddings of many articles (article 454 is highlighted here in blue as an example) are very similar, some of the articles' embeddings are not preserved well across modality (article 366 is highlighted here in red as an example; note that the English graph with shortest path distance differs significantly from the other three modalities for this article).


\subsubsection{Zebrafish brains}
\label{S:zeebs}
In \cite{prevedel2014simultaneous}, the authors combined Light-Field Deconvolution Microscopy and pan-neuronal expression of GCaMP, a fluorescent calcium indicator that serves as a proxy for neuronal activity, to produce a time series of whole-brain zebrafish neuronal activity at near single neuron resolution.
The data consists of 5000 realizations of a multivariate time series $\{Z^{(t)}\}_{t=1}^{5000}$ with  
$Z^{(t)}\in\mathbb{R}^{5379}$ for all $t$, where for each $i\in[5379],\,Z^{(t)}(i)\in\mathbb{R}$ represents the activity of neuron $i$ at time $t$.
Each time frame $[t,t+1)$ is $1/20$ of a second; i.e. the data was collected at 20 Hz.
After preprocessing the data and removing some artificial edge neurons, we are left with $Z^{(t)}\in\mathbb{R}^{5105}$ for each of $t=1,2,\ldots,5000$.

  \begin{figure}[t!]
\centering
\includegraphics[width=.8\textwidth]{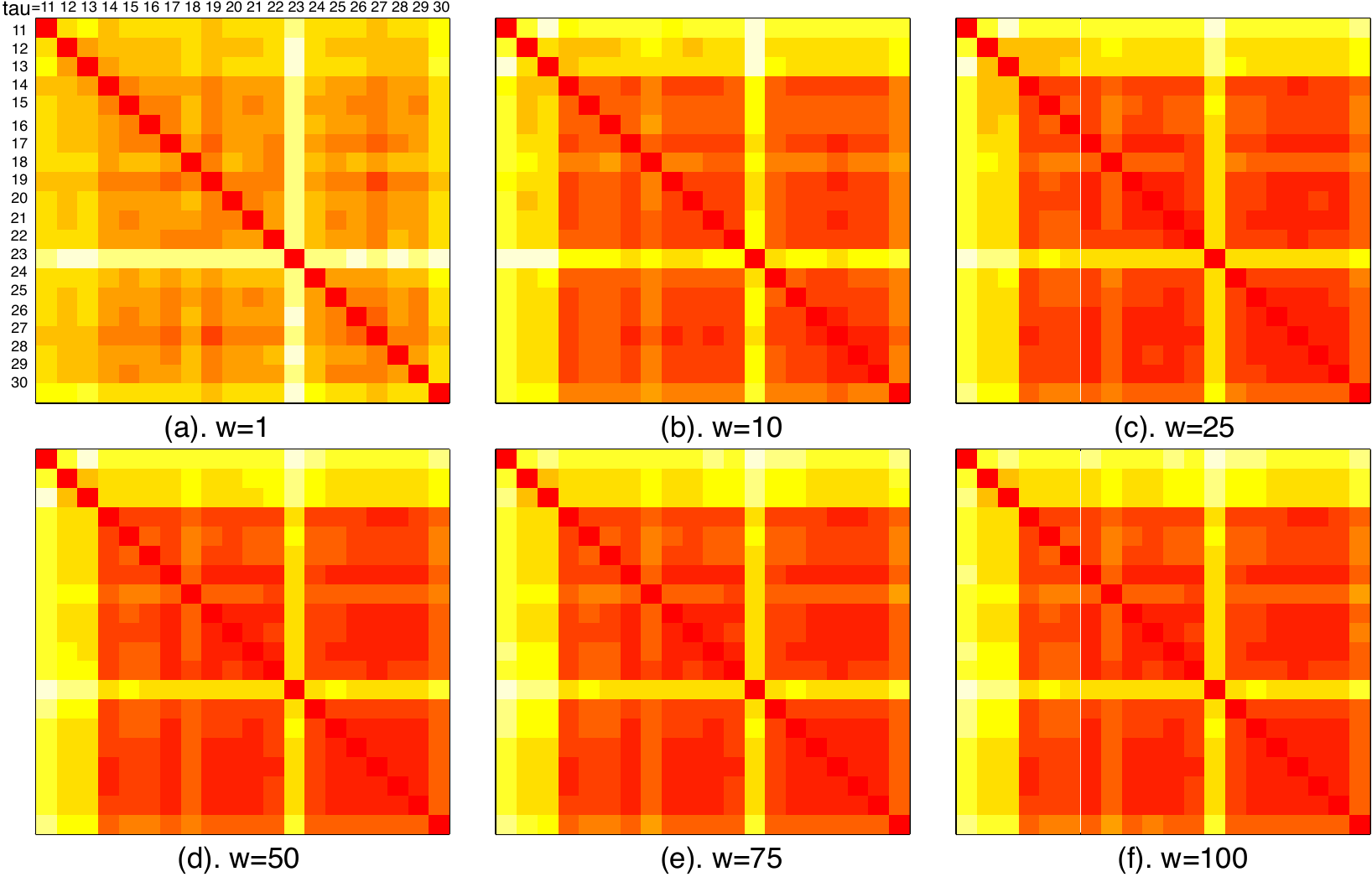}
\caption{Heatmaps of the Frobenius norm differences between the $20$ zebrafish neuron embeddings $\{\bX^{(\gt)}\}_{\gt=11}^{30}$ obtained by fJOFC over a range of $w$'s.
Each heatmap is a $20\times 20$ grid, where the intensity of the $i,j$-th entry indicates the difference between the embeddings of the $n^*$ fish neurons at times $\gt=i$ and $\gt=j$; more red indicates less difference in the embedded space and white indicating very different embeddings.  Note the anomalous point at $\gt=23.$}
\label{fig:fishheat}
\end{figure}

Binning the time stamps into $100$ overlapping periods of $5$ seconds (so that for each $\gt\in[100],$ bin $\gt$ consists of the matrix of observations 
$${\bf Z}^{(\gt)}=[Z^{(50(\gt-1)+1)}|\cdots|Z^{(50(\gt+1)}]=\left[
({\bf Z}^{(\tau)}_1)^\top| 
({\bf Z}^{(\tau)}_2)^\top|
\cdots|
({\bf Z}^{(\tau)}_{5105})^\top\right]^\top\in\mathbb{R}^{5105\times 100}),$$ we compute a time series of 100 dissimilarity matrices $\{\D^{(\gt)}\}_{\gt=1}^{100}$ as follows.  For each $\gt$, we compute the thresholded 
correlation matrix $D^{(\gt)}\in\mathbb{R}^{5105\times 5105}$
with $$D^{(\gt)}_{i,j}=\mathbbm{1}\{|\text{corr}(\bZ^{(\gt)}_i,\bZ^{(\gt)}_j)|>0.7\}$$ (where the threshold $0.7$ was chosen to ensure sufficient sparsity in the resulting $D^{(\gt)}$'s).
These correlation matrices are then transformed to dissimilarity matrices $\{\D^{(\gt)}\}_{\gt=1}^{100}$ by defining
$$\D^{(\gt)}_{i,j}=1-\frac{|N_\gt(i)\cap N_\gt(i)|}{|N_\gt(i)\cup N_\gt(i)|},$$
where $N_\gt(i)$ is the neighborhood of neuron $i$ in $D^{(\gt)}$ viewed as a graph.

Initial change point detection analysis, analogous to that in \cite{fishes}, indicated that there was an anomaly in the neural correlations at time $\gt^*=23$ and identified $n^*=469$ neurons responsible for this anomaly.   
To explore this further, we use fJOFC to embed a portion of the time series (from times $\gt=11$ to $\gt=30$) obtaining the configuration 
$$\bX^\top=[(\bX^{(11)})^\top|(\bX^{(12)})^\top|\cdots|(\bX^{(30)})^\top]^\top.$$
If there is an anomaly in the activity of the $n^*$ neurons at $\gt^*=23$, this should be evinced by $\bX^{(23)}$ significantly differing from $\bX^{(\gt)}$ for $\gt\neq 23$, as seen in Figures \ref{fig:fishheat} and \ref{fig:fishplot}.
Moreover, the embedding can also inform the {\it structure} of the anomaly, as we can identify the change in structure within the $n^*=469$ neurons which is responsible for the anomaly in the embedded space; see Figure \ref{fig:fishmove}.
Below, we expound on the details of our embedding procedure and findings.
  \begin{figure}[t!]
\centering
\includegraphics[width=.8\textwidth]{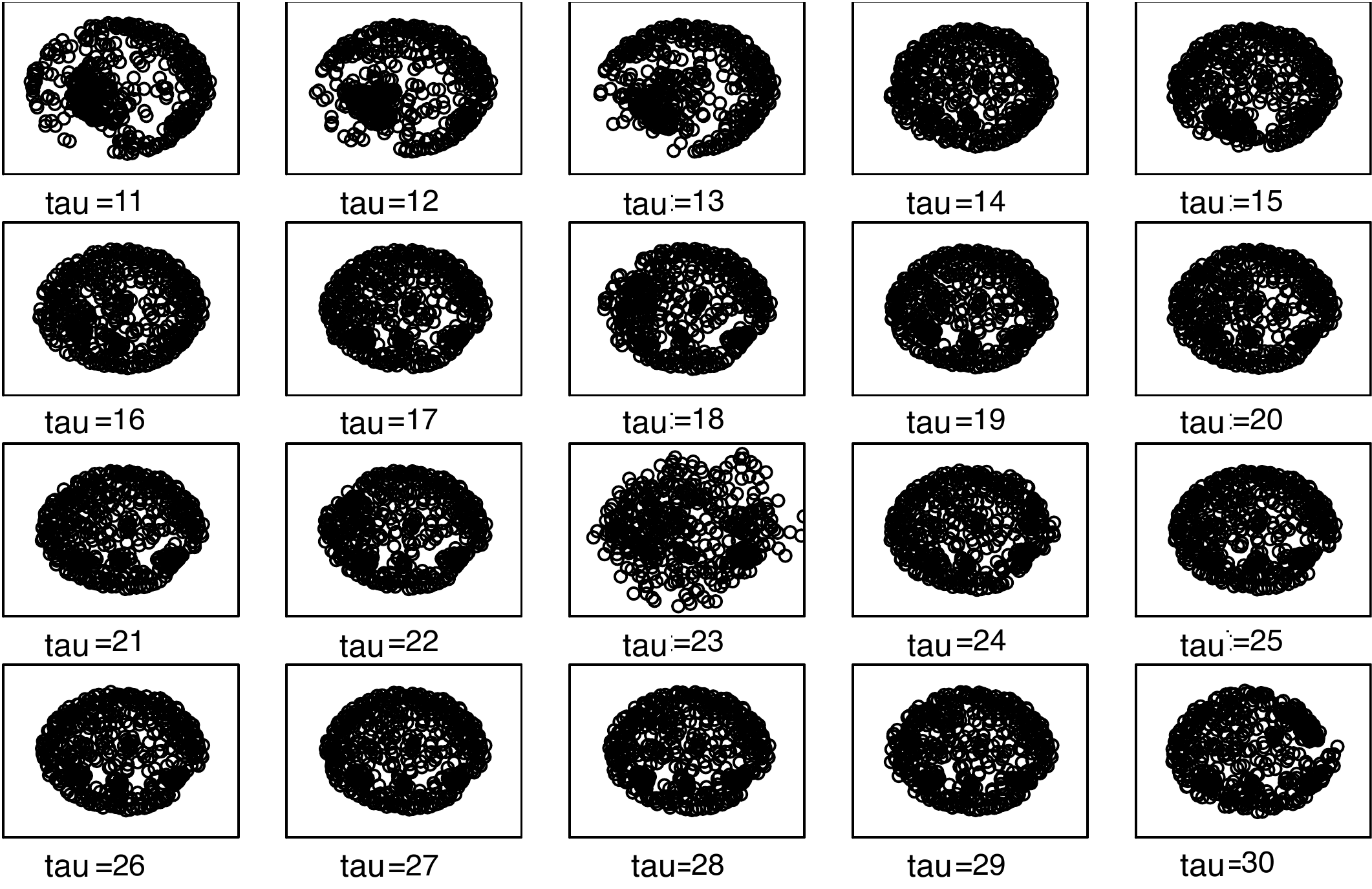}
\caption{Embeddings of the $m=20$ elements of the time-series $\{\widetilde\D^{(\gt)}\}_{\gt=11}^{30}$ into $\mathbb{R}^2$ obtained via fJOFC with $w=10$.
Each of the $20$ plots is on the same set of axes.
Note the anomaly at $\gt=23$.}
\label{fig:fishplot}
\end{figure}

Restricting the full dissimilarities to the $n^*$ identified anomalous neurons---yielding a times series $\{\widetilde\D^{(\gt)}\}_{\gt=1}^{100}$ of $100$  dissimilarities in $\mathbb{R}^{469\times469}$---we first embed the $m=20$ elements of $\{\widetilde\D^{(\gt)}\}_{\gt=11}^{30}$ into $\mathbb{R}^2$.
To test if there is an anomaly at $\gt=23$, we next compute the Frobenius norm differences between the $20$ embeddings $\{\bX^{(\gt)}\}_{\gt=11}^{30}$ in the configuration.  
Results are summarized in Figure \ref{fig:fishheat}, where we plot a heatmap of the Frobenius norm differences between the $\{\bX^{(\gt)}\}_{\gt=11}^{30}$ over a range of $w$'s
(plots of the 2-dimensional fJOFC embeddings with $w=10$ across $\gt=11,12,\ldots,30$ are displayed in Figure \ref{fig:fishplot}).
Each heatmap is a $20\times 20$ grid, where the intensity of the $i,j$-th entry indicates the difference between the embeddings of the $n^*$ fish neurons in $\bX^{(i)}$ and $\bX^{(j)}$; more red indicates less difference in the embedded space and white indicating very different embeddings. 
We see that, across the range of $w$'s, there is a significant anomaly in the embedding at $\gt=23$.  This both confirms the initial findings of an anomaly at $\gt=23$ and demonstrates the potential robustness of this anomaly-detection procedure to misspecified $w$.
We also note that the embeddings at times $\gt=11,12,13$ are significantly different from the embeddings at all other times.
Further analysis is needed to determine if this is neuroscientifically significant or a data collect/algorithmic artifact.
We lastly note that this embedding ran in $\approx 1.5$ hours using fJOFC run in serial and over $20$ hours using JOFC, again showing the dramatic speedup of our fJOFC procedure.

To further understand the structure of this anomaly, we plot the change in the embeddings from times 21--22, times 22--23, times 23--24, and times 24--25 in Figure \ref{fig:fishmove} (so that there are $2n^*$ points in each panel).
In the figure, the neurons in the configuration at time 23 are displayed as red points, with neurons in the configuration at other times displayed as black points.
For each individual neuron, the movement in the configuration from times $\gt$ to $\gt+1$ are highlighted with blue lines; i.e., there is a line connecting the position of the neuron at time $\gt$ to its position at time $\gt+1$. 
From this figure, we can identify the groups of neurons whose change in activity is responsible for the anomaly.
Again, further analysis is necessary to determine the potential neuroscientific significance of these neurons' activity.

\begin{figure}[t!]
\centering
\includegraphics[width=.8\textwidth]{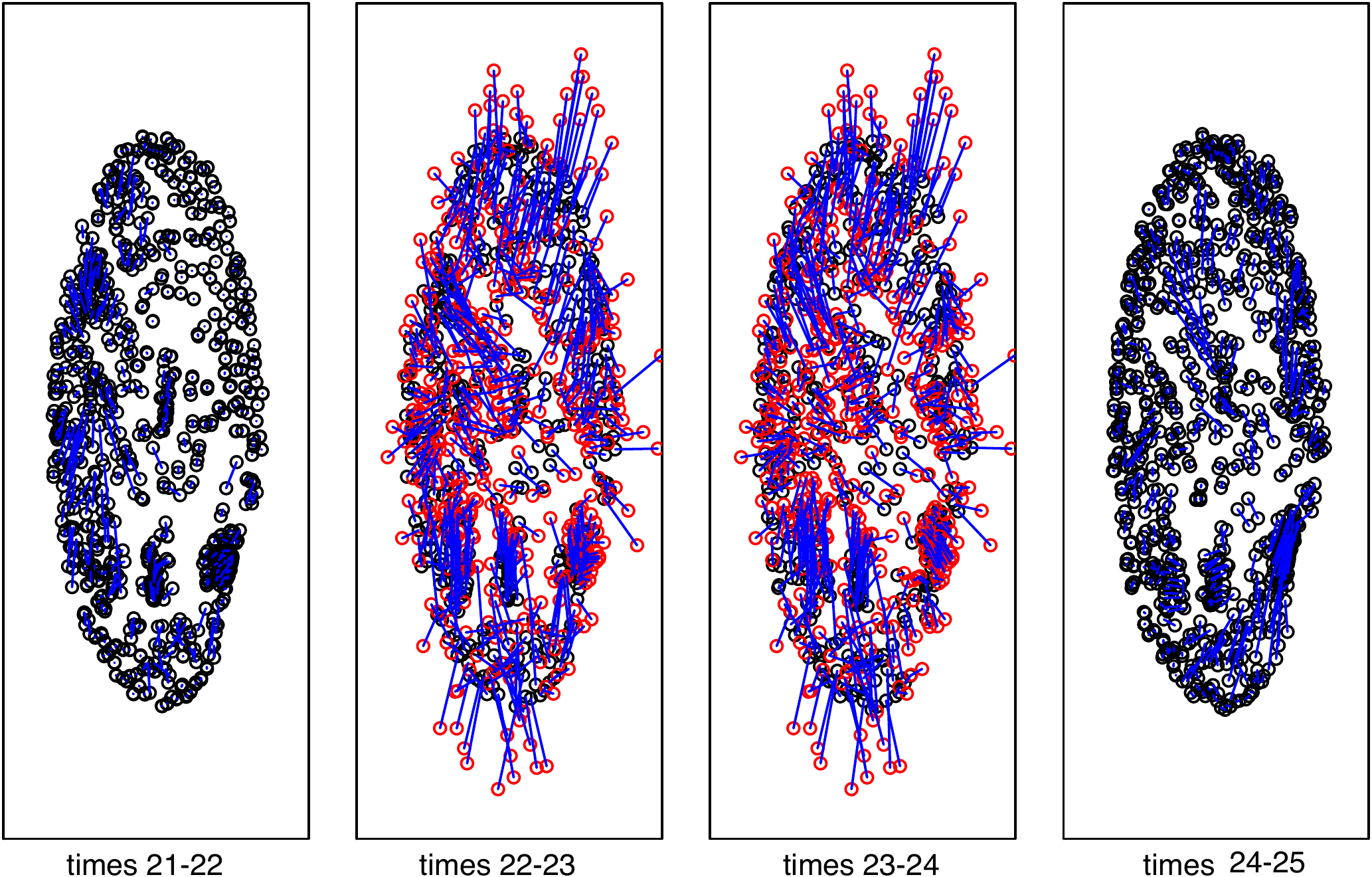}
\caption{Plot of the change in the embeddings from times 21--22, times 22--23, times 23--24, and times 24--25 (so that there are $2n^*$ points in each panel).
In the figure, the neurons in the configuration at time 23 are displayed as red points, with neurons in the configuration at other times displayed as black points.
For each individual neuron, the movement in the configuration from times $\gt$ to $\gt+1$ are highlighted with blue lines; i.e., there is a line connecting the position of the neuron at time $\gt$ to its position at time $\gt+1$.}
\label{fig:fishmove}
\end{figure}

\section{Conclusion}
The JOFC algorithm has proven to be a valuable and adaptable tool for a variety of inference tasks (e.g., graph matching \citep{sgmjofc}; hypothesis testing \citep{JOFC}; joint classification \citep{sunpriebe2013}; among others).  
The key capability enabled by our fJOFC algorithm (both in-sample and out-of-sample) versus the JOFC algorithm is enhanced scalability in $m$ and $n$; indeed, for a fixed $n$, we see a factor of $m$ speedup over the JOFC algorithm, and for a fixed $m$ we see a factor of $n$ speed up achieved by fJOFC.
Additionally, the out-of-sample fJOFC procedure is shown to have linear runtime in $n$.
Combined with sparse dissimilarity representations of very large data sets, this capability to simultaneously embed many different large dissimilarities, both in and out-of-sample, 
enables the
complex structure of the data to more easily be interrogated,
leading to potentially significant discoveries heretofore beyond our grasp.

While the sequential Guttman transforms computed in Algorithm \ref{alg:1} are only guaranteed to converge to a stationary configuration, because the sequence of raw stress values is decreasing, in practice they will typically converge to a local minimizer of ${\bf \sigma(X)}$.
Note that, in most cases, the local convergence rate of the iterative Guttman transforms is linear, see \cite{smacof2}.  
In practice, the sequential Guttman transforms often exhibit good global properties, and only a few iterations are required to obtain a sufficiently good suboptimal embedding, see \cite{kearsley1995solution}.  
Analyzing these global properties and/or modifying fJOFC to accelerate the linear convergence---for example, by incorporating relaxed updates in the iterative majorization as in \cite{de1980multidimensional}---are essential next steps for further scaling fJOFC to very big data.

\appendix
\section{Derivation of $\bf L^\dagger$}
\label{app:1}
In this section, we collect supporting results to derive the desired form of $\bf L^\dagger$.  
We first prove that $\bl^\dagger$ can be realized via ${\bf L}^{\dagger}=\left({\bf L}+\frac{1}{mn}J_{mn}\right)^{-1}-\frac{1}{mn}J_{mn}$,
\begin{proposition}
\label{prop:pseudoinv}
Let $\bW$ be any symmetric weight matrix in $\mathbb{R}^{mn\times mn}$.  If $\bl$ is the combinatorial Laplacian of $\bW$, then $\bl$ can be equivalently realized via 
\begin{equation}
\label{eq:pseudoinv}
{\bf L}^{\dagger}=\left({\bf L}+\frac{1}{mn}J_{mn}\right)^{-1}-\frac{1}{mn}J_{mn}.
\end{equation}
\end{proposition}
\begin{proof}
The proof is straightforward linear algebra, but we include it here for completeness.
We first note that $J_{mn}\bl={\bf L}J_{mn}=0$, so that 
$$\left(\bl+\frac{1}{mn}J_{mn}\right)J_{mn}=J_{mn}=J_{mn}\left(\bl+\frac{1}{mn}J_{mn}\right).$$
We then calculate
\begin{align*}
\bl\left[\left({\bf L}+\frac{1}{mn}J_{mn}\right)^{-1}-\frac{1}{mn}J_{mn}\right]\bl&=\bl\left({\bf L}+\frac{1}{mn}J_{mn}\right)^{-1}\bl\\
&=\bl\left({\bf L}+\frac{1}{mn}J_{mn}\right)^{-1}\left(\bl+\frac{1}{mn}J_{mn}-\frac{1}{mn}J_{mn}\right)\\
&=\bl\left(I_{mn}-\frac{1}{mn}J_{mn}\right)=\bl;
\end{align*}
and
\begin{align*}
&\left[\left({\bf L}+\frac{1}{mn}J_{mn}\right)^{-1}-\frac{1}{mn}J_{mn}\right]\bl\left[\left({\bf L}+\frac{1}{mn}J_{mn}\right)^{-1}-\frac{1}{mn}J_{mn}\right]\\
&=\left[\left({\bf L}+\frac{1}{mn}J_{mn}\right)^{-1}-\frac{1}{mn}J_{mn}\right]\left(\bl+\frac{1}{mn}J_{mn}-\frac{1}{mn}J_{mn}\right)\left[\left({\bf L}+\frac{1}{mn}J_{mn}\right)^{-1}-\frac{1}{mn}J_{mn}\right]\\
&=\left[I_{mn}-2\frac{1}{mn}J_{mn}+\frac{1}{mn}J_{mn} \right]\left[\left({\bf L}+\frac{1}{mn}J_{mn}\right)^{-1}-\frac{1}{mn}J_{mn}\right]=
\left({\bf L}+\frac{1}{mn}J_{mn}\right)^{-1}-\frac{1}{mn}J_{mn};
\end{align*}
and $\left[\left({\bf L}+\frac{1}{mn}J_{mn}\right)^{-1}-\frac{1}{mn}J_{mn}\right]\bl=I_{mn}-\frac{1}{mn}J_{mn}=\bl\left[\left({\bf L}+\frac{1}{mn}J_{mn}\right)^{-1}-\frac{1}{mn}J_{mn}\right]$ is Hermitian.  It follows that
$\bl^\dagger=\left({\bf L}+\frac{1}{mn}J_{mn}\right)^{-1}-\frac{1}{mn}J_{mn}$ as desired.
\end{proof}

We have that
the combinatorial Laplacian of ${\bf W}$ is given by ${\bf L}=\mathcal W\otimes I_n-\text{diag}(w_{i,i})\otimes J_n$.
It follows that 
$${\bf L}+\frac{1}{mn}J_{mn}
=\mathcal W\otimes I_n+\left(\frac{1}{mn}J_{m}-\text{diag}(w_{i,i})\right)\otimes J_n.$$
Proposing that $({\bf L}+\frac{1}{mn}J_{mn})^{-1}$ is of the form $\mathcal{V}\otimes I_n+\mathcal{Z}\otimes J_n,$ we arrive at the following.
\begin{theorem}
\label{thm:1}
With notation as above, let $\bW$ be a weight matrix of the form of Eq. (\ref{eq:W2}), and assume that $w_{i,i}>0$ for all $i\in\{1,2,\ldots,m\}$.  Let $\bl$ be the combinatorial Laplacian of $\bW$, then 
$$\left({\bf L}+\frac{1}{mn}J_{mn}\right)^{-1}=\left(\mathcal W\otimes I_n+\left(\frac{1}{mn}J_{m}-\text{diag}(w_{i,i})\right)\otimes J_n\right)^{-1}=\mathcal{V}\otimes I_n+\mathcal{Z}\otimes J_n,$$
where $
\mathcal V=\mathcal W^{-1}$
and ${}
\mathcal Z=-\left(\mathcal{W}+n\left(\frac{1}{mn}J_{m}-\text{diag}(w_{i,i})\right)\right)^{-1}\left(\frac{1}{mn}J_{m}-\text{diag}(w_{i,i})\right)\mathcal{W}^{-1}.$
\end{theorem}
\begin{proof}
First note that the assumption on $\{w_{i,i}\}_{i=1}^m$ assures that $\mathcal W$ is strictly diagonally dominant and is therefore invertible.
If the proposed form, $({\bf L}+\frac{1}{mn}J_{mn})^{-1}=\mathcal{V}\otimes I_n+\mathcal{Z}\otimes J_n,$ is correct then 
\begin{align*}
I_{mn}&=\left({\bf L}+\frac{1}{mn}J_{mn}\right)\left({\bf L}+\frac{1}{mn}J_{mn}\right)^{-1}\\
&=\left(\mathcal W\otimes I_n+\left(\frac{1}{mn}J_{m}-\text{diag}(w_{i,i})\right)\otimes J_n\right)\left(\mathcal{V}\otimes I_n+\mathcal{Z}\otimes J_n,\right)\\
&=(\mathcal{W}\mathcal{V})\otimes I_n+\left(\left(\mathcal{W}+n\left(\frac{1}{mn}J_{m}-\text{diag}(w_{i,i})\right)\right)\mathcal{Z}+\left(\frac{1}{mn}J_{m}-\text{diag}(w_{i,i})\right)\mathcal{V}\right)\otimes J_n
\end{align*}
From this, the desired forms of $\mathcal{V}$ and $\mathcal{Z}$ follow immediately.
\end{proof}

From Theorem \ref{thm:1}, the following Corollary is immediate:
\begin{corollary}
\label{cor:1}
With notation as above, let $\bW$ be a weight matrix of the form of Eq. (\ref{eq:W2}), and assume that $w_{i,i}>0$ for all $i\in\{1,2,\ldots,m\}$.  Let $\bl$ be the combinatorial Laplacian of $\bW$, then
$$\bl^\dagger=\mathcal{W}^{-1}\otimes I_n+\left[-\left(\mathcal{W}+n\left(\frac{J_m}{mn}-\text{diag}(w_{i,i})\right)\right)^{-1}\left(\frac{J_m}{mn}-\text{diag}(w_{i,i})\right)\mathcal{W}^{-1}-\frac{J_m}{mn} \right]\otimes J_n.$$
\end{corollary}

\bibliographystyle{asa}
\bibliography{refbib}
\end{document}